\documentclass[letterpaper, 10 pt, conference, onecolumn]{ieeeconf}  % Comment this line out if you need a4paper

\IEEEoverridecommandlockouts                              % This command is only needed if 
\usepackage{amsmath,graphicx,multicol}
\usepackage{amsfonts}
\usepackage{color}
\usepackage{subcaption}
\usepackage{bbm}
\usepackage{cite}
\usepackage{amssymb}
\usepackage{algorithm}
\usepackage{algorithmic}
\usepackage{tabularx,booktabs}
\usepackage[utf8]{inputenc}
\usepackage[english]{babel}

\usepackage{amsthm}
\usepackage{bm}
\usepackage{hyperref}
\usepackage{cleveref}
\usepackage[font={footnotesize}]{caption}
\usepackage{url}
\usepackage{mathtools}

\newtheorem{assumption}[]{Assumption}
\newtheorem{remark}{\textbf{Remark}}

\newtheorem{theorem}{Theorem}[]
\newtheorem{corollary}{Corollary}[theorem]
\newtheorem{lemma}[]{Lemma}

\DeclareMathOperator{\E}{\mathbb{E}}

\def\E{\mathbb{E}}

\newcommand{\RNum}[1]{\uppercase\expandafter{\romannumeral #1\relax}}
\makeatletter
\renewcommand{\fnum@figure}{Fig.~\thefigure}
\makeatother

\title{\LARGE \bf
Improved Convergence Analysis and SNR Control Strategies for Federated Learning in the Presence of Noise
}

\author{Antesh Upadhyay and Abolfazl Hashemi% <-this % stops a space
\thanks{Antesh Upadhyay and Abolfazl Hashemi are with the School of Electrical and Computer Engineering, Purdue University, West Lafayette, IN 47907, USA.}% <-this % stops a space
}

\begin{document}

\maketitle
\thispagestyle{empty}
\pagestyle{empty}

%%%%%%%%%%%%%%%%%%%%%%%%%%%%%%%%%%%%%%%%%%%%%%%%%%%%%%%%%%%%%%%%%%%%%%%%%%%%%%%%
\begin{abstract}
We propose an improved convergence analysis technique that characterizes the distributed learning paradigm of federated learning (FL) with imperfect/noisy uplink and downlink communications. Such imperfect communication scenarios arise in the practical deployment of FL in emerging communication systems and protocols. The analysis developed in this paper demonstrates, for the first time, that there is an asymmetry in the detrimental effects of uplink and downlink communications in FL. In particular, the adverse effect of the downlink noise is more severe on the convergence of FL algorithms. Using this insight, we propose improved Signal-to-Noise (SNR) control strategies that, discarding the negligible higher-order terms, lead to a similar convergence rate for FL as in the case of a perfect, noise-free communication channel while incurring significantly less power resources compared to existing solutions. In particular, we establish that to maintain the $\mathcal{O}(\frac{1}{\sqrt{K}})$ rate of convergence like in the case of noise-free FL, we need to scale down the uplink and downlink noise by  $\Omega({\sqrt{k}})$ and $\Omega({k})$ respectively, where $k$ denotes the communication round, $k=1,\dots, K$. Our theoretical result  is further characterized by two major benefits: firstly, it does not assume the somewhat unrealistic assumption of \textit{bounded client dissimilarity}, and secondly, it only requires smooth non-convex loss functions, a function class better suited for modern machine learning and deep learning models.  We also perform extensive empirical analysis to verify the validity of our theoretical findings.
\end{abstract}

%%%%%%%%%%%%%%%%%%%%%%%%%%%%%%%%%%%%%%%%%%%%%%%%%%%%%%%%%%%%%%%%%%%%%%%%%%%%%%%%
\section{Introduction}
\label{sec:introduction}
\PARstart{T}{he} advancements in the field of Machine Learning (ML) are attributable to the increasing ability to generate and process data from various edge devices such as sensors, mobile phones, and the internet of things (IoT) devices. However, traditional ML approaches that rely on storing data on a server can be problematic in terms of the privacy of the user and scalability. To address these issues, we see the eminent shift towards distributed and collaborative ML approaches, such as consensus-based distributed optimization and Federated learning (FL) \cite{ren2005consensus,jadbabaie2003coordination,nedic2009distributed,nedic2014distributed,nedic2016stochastic,kairouz2019advances, lim2020federated}, which allow for learning to occur without the need for central data storage. In this paper, our primary focus is on the FL setting. The algorithm popularly known as Federated Averaging (\texttt{FedAvg}) presented in \cite{mcmahan2017communication}, provided the foundation for FL. Breaking the quintessential model of traditional ML, \texttt{FedAvg}, preserves the privacy of the agent (typically referred to as clients in FL) by allowing them to retain their data. In \texttt{FedAvg}, a set of agents (referred to as clients in FL), based on their local data,  perform the Stochastic Gradient Descent (SGD) iteratively for a certain number of local steps and then transmits their updated model parameters to a central server, which then averages these updates and, in turn, updates the global model. Iterative communication between servers and clients and the collaborative nature of FL shows the importance of communication vis-\`a-vis, \texttt{FedAvg} and other FL and consensus-based methods, and it has been an active area of research in terms of improving the efficiency and resiliency of such algorithms \cite{konevcny2016federated,das2021privacy,tang2018communication,wen2017terngrad,zhang2017zipml,savas2021physical,das2022faster,raja2021communication}.
\subsection{Realted Works}
Several recent results, where improving communication efficiency is the core, focus mainly on reducing the number of communication rounds \cite{mcmahan2017communication,li2020federated}, or the size of information during transmission \cite{reisizadeh2020fedpaq,du2020high,zheng2020design,hashemi2021benefits,chen2021communication,chen2021decentralized}. However, in most of these studies, the process of communication from the server to clients (\textit{downlink}) and then from clients to the server (\textit{uplink}), a perfect communication link is often assumed. Now, some literature investigates the impact of having a noisy transmission channel but only studies the effect of noisy uplink transmission \cite{amiri2020federated,zhu2019broadband,xia2021fast,sery2021over,guo2020analog}. However, only a few articles in the literature deal with the impact of only downlink noise or both noises \cite{wei2022federated}. A major consideration among all these works is their somewhat restrictive  assumptions that typically are not satisfied in practical settings or are hard to verify. For instance, in \cite{amiri2020federated}, where they analyze the effect of downlink noise, they assume a perfect uplink communication channel. 
% {\color{black}why is the next sentence relevant to this paper? discuss a bit more or get rid of it}
Additionally, existing research that studies both uplink and downlink noise focuses on the modification of the training of ML models. Reference \cite{ang2020robust} aims to counter the effect of noise by modifying the loss function to consider the addition of noise as a regularizer. Similarly \cite{tang2019doublesqueeze,yu2019double,chen2020scalecom, hashemi2021benefits} focus on compressing the gradients to counter the effect of noisy transmission channels. Since compression inherently adds noise to the message communicated, it poses an adverse impact on model convergence. The limitations of these works serve as our primary motivation as we aim to understand the impact of the uplink and downlink communication on the performance of \texttt{FedAvg} by developing an improved analysis to reduce the burden of intensive power consumption while relaxing the assumptions of convexity and bounded client dissimilarity (BCD) required by the existing works.

Recently, \cite{wei2022federated}  studies the impact of both uplink and downlink noises with restrictive assumptions of strong-convexity and Bounded Client Dissimilarity (BCD) \cite{karimireddy2020scaffold}. To avoid the client-drift\cite{karimireddy2020scaffold}, a standard assumption used in FL is BCD (refer to \cref{eq:bcd}). This drift occurs due to multiple local SGD updates on clients with non-IID data distribution, which prohibits the algorithm from converging to the global optimum. Nevertheless, the result of \cite{wei2022federated} has an important shortcoming: the analysis is not tight due to which while the model converges, the dominant terms in the convergence error depend on noise characteristics and as a result, the Signal-to-Noise Ratio (SNR) scaling policy requires more power compared to our results.
\subsection{Contributions}
The contributions of our work are motivated towards mitigating the restrictions imposed in previous literature. Unlike the previous work, \cite{wei2022federated}, we propose an analysis of smooth non-convex \texttt{FedAvg} with noisy (both uplink and downlink) communication channels and without the BCD assumption. 
% To avoid the client-drift\cite{karimireddy2020scaffold}, a standard assumption used in FL is BCD (\cref{eq:bcd}. This drift occurs due to multiple local SGD updates on clients with non-IID data distribution, which makes the algorithm to converge an inferior optimum. 
We leverage the non-negativity of the typical loss functions in optimization/ML and their smoothness in conjunction with a novel sampling technique to avoid using BCD while establishing  our improved convergence results. 
% To perform a convergence analysis of \texttt{FedAvg}, with a perfect communication channel is hard, but incorporating simultaneous noise in both upload and download cases makes the process much more difficult. So, this requires careful manipulation of terms and subtle use of mathematical inequalities.
We present the results of our analysis in Theorem \ref{thm-noisy-fedavg} and \Cref{coro-noisy-fedavg}, which shows that the effect of downlink noise, i.e., $\mathcal{O}(1)$, is more degrading than uplink noise, i.e., $\mathcal{O}(\frac{1}{\sqrt{K}})$ where $K$ is the number of communication rounds. Hence, following these results, we draw an inference that as long as we control the effect of downlink and uplink noise such that they do not dominate the inherent noisy communication aspect of SGD, the convergence of the model can be achieved while limiting the adverse effect of noise to negligible higher-order terms. In particular, we demonstrate both theoretically and empirically that in order to maintain the convergence rate of $\mathcal{O}(\frac{1}{\sqrt{K}})$ for the case of noise-free \texttt{FedAvg}, we need to scale down the downlink noise by $\Omega({k})$ and uplink noise by $\Omega({\sqrt{k}})$, or, equivalently scale the downlink noise by $\Omega(\frac{1}{k})$ and uplink noise by $\Omega(\frac{1}{{\sqrt{k}}})$. These scaling rates ensure that the noise appears as a higher-order term, not as a dominant term\footnote{Refer \cref{thm-noisy-fedavg} and \cref{coro-noisy-fedavg} for more details about the effect of noises.}.

To summarize, the contributions of this paper are as follows:
\begin{itemize}
    \item We provide an improved analysis alongside the complete proof of FL under the presence of uplink and downlink noise without using any constraining assumption which results in tighter convergence analysis. 
    \item In reference to \cref{coro-noisy-fedavg}, we provide empirical results that show the asymmetric effect of both uplink and downlink noises. We provide plots that verify that for a constant number of communication rounds, uplink noise scales as $\mathcal{O}(\frac{1}{E^{2}\sqrt{r}})$, while the term corresponding to downlink noise is $\mathcal{O}(1)$. Here $r$ denotes the number of clients participating in each round.
    \item In this paper, we present a method for controlling the SNR ratio in order to mitigate the deleterious effects of noise on both the uplink and downlink communication channels. The proposed scaling policy is more accommodating and robust towards higher noise concentrations as well. Our analysis demonstrates that this approach leads to improved performance.
    \item We provide empirical results on both synthetic and real-world deep learning experiments on datasets such as MNIST, Fashion-MNIST, CIFAR-10, and FEMNIST datasets, to establish the efficacy and validity of the proposed technique. 
\end{itemize}

\section{Preliminaries and System-Model}\label{sec:prelim}
The setting of the problem follows the traditional FL scenario presented in \cite{mcmahan2017communication} (see also \Cref{fig:noisy-fedavg-diagram}).
In a standard FL setting, we have a central server and a set of $n$ clients, each having their local training data. The $i^{th}$ client stores their local data sampled from a distribution $\mathcal{D}_i$. The central server aims to train a machine learning model on the client's local data, parameterized by $\bm{w} \in \mathbb{R}^{d}$. Then, $\bm{f}_i(\bm{w})$ is the expected loss over a sample $\bm{x}$ drawn from $\mathcal{D}_i$ with respect to a loss function $\ell$ for the $i^{th}$ client. The primary objective of the central server is to minimize the loss $\bm{f}(\bm{w})$ over $n$ clients, i.e., 
\begin{equation}
    \label{eq:objective-function}
    \bm{f}(\bm{w}) := \frac{1}{n}\sum_{i=1}^{n}\bm{f}_i(\bm{w})\  \text{\&} \ \bm{f}_i(\bm{w}) = \mathbb{E}_{\bm{x}\sim \mathcal{D}_{i}}[\ell(\bm{x}, \bm{w})].
\end{equation}

Also, to emulate an FL setting in practice, we consider partial client participation, i.e., a set of $r$ clients chosen uniformly at random without replacement
from a set of $n$ clients, whereas in the case of full participation, $r = n$. Such an assumption is motivated by the consideration that the clients may have limited communication capabilities and not all will be able to collaborate at every communication round. We assume that these clients have access to the unbiased stochastic gradient of their individual losses which is denoted by $\widetilde{\nabla}f_i(\bm{w};\mathcal{B})$ computed at $\bm{w}$ over a batch of samples $\mathcal{B}$. In addition, $\bm{K}$ denotes the communication rounds, and $\bm{E}$ represents the number of local iterations for each communication round.

The FL process can be thought of as an iterative, three-step pipeline: 1) global model update from the central server to the clients over a noisy channel, i.e., noisy downlink communication, 2) client-level computation, and 3) sending updated model parameters from the clients to the server over a noisy channel, i.e, noisy uplink communication. We will discuss these steps next.
\subsection{Noisy downlink communication}
The central server sends the global model parameter, $\bm{w}_k$, to the set of $r$ clients denoted by $\mathcal{S}_k$ chosen uniformly at random without replacement. Now due to disturbances and distortion in the communication channel, these clients receive a noisy version of the global model parameter, i.e.,
\begin{equation}
    \label{eq:noisy-global-update}
    \bm{w}_{k, 0}^{(i)} = \bm{w}_k + \bm{\nu}_k^{(i)}\text{,}
\end{equation}
where $\bm{\nu}_k^{(i)} \in \mathbb{R}^{d}$ is the zero mean random downlink noise and $\bm{w}_{k, 0}^{(i)}$ is the received model to the $i^{th}$ client.
Subsequently, the SNR we get for the $i^{th}$ client for $k^{th}$ downlink communication round can be written as,
\begin{equation}
    \label{eq:SNR-downlink}
    \mathrm{SNR}_{k, (i)}^D = \frac{\mathbb{E}[||\bm{w}_k||^2]}{\mathbb{E}[||\bm{\nu}_k^{(i)}||^2]}.
\end{equation}
Since we assumed that $\bm{\nu}_k^{(i)}$ is a  zero mean noise, the variance can be written as: 
\begin{equation}
    \label{eq:var-downlink}
    \bm{N}^{2}_{k,i}:=\mathbb{E}\Big[\Big\|\bm{\nu}_k^{(i)}\Big\|^2\Big].
\end{equation}
\begin{figure}[t]
\includegraphics[width=1\linewidth]{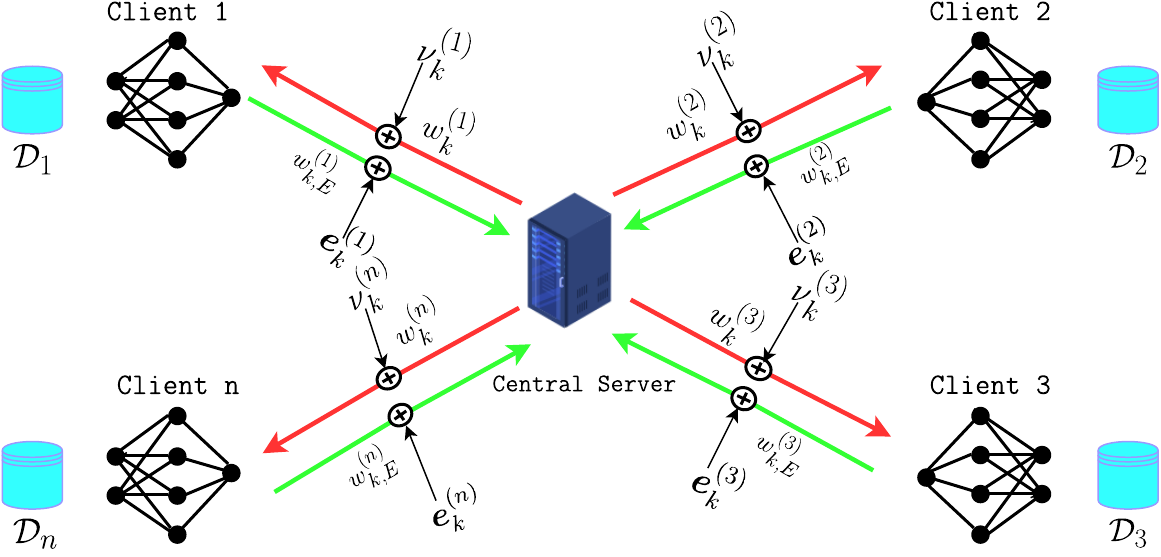}
\caption{Problem Setting: an FL system with uplink and downlink communication noise.}
\label{fig:noisy-fedavg-diagram}
\vspace{-4mm}
\end{figure}
\subsection{Client level computation}
Each client performs a local computation on its data using the updated noisy global model parameter. We use mini-batch SGD for training the model and updating the weights iteratively. This can be referenced from lines 6 to 9 in Algorithm \ref{alg:noisy-fed-avg} and written as,
\begin{align}
    \label{eq:noisy-client-level-computation}
    \bm{w}_{k,\tau+1}^{(i)} = \bm{w}_{k,\tau}^{(i)} - \eta_{k} \widetilde{\nabla} f_i(\bm{w}_{k,\tau}^{(i)};\mathcal{B}_{k,\tau}^{(i)}),
    \\
    \nonumber
    \forall \tau = 0,1, \dots, E-1,
\end{align}
where $\mathcal{B}_{k,\tau}^{(i)}$ represents the random batch of samples in client $i$ for $\tau^{th}$ local iteration. 
\subsection{Noisy uplink communication}
After the local computation, the clients in $\mathcal{S}_k$ send their local model to the central server. Similar to the downlink case, due to disturbances and distortion in the communication channel, the central server receives a noisy version of local weights which can be seen from line 10 in Algorithm \ref{alg:noisy-fed-avg} and is formulated as,
\begin{equation}
    \label{eq:noise-model-update}
    \bm{w}_{k,0}^{(i)} - \bm{w}_{k, E}^{(i)} + \bm{e}_{k}^{(i)},
\end{equation}
where $\bm{e}_k^{(i)} \in \mathbb{R}^{d}$ is a zero mean random noise. Subsequently, the SNR we get for the $i^{th}$ client for $k^{th}$ uplink communication round can be written as,
\begin{equation}
    \label{eq:SNR-uplink}
    \mathrm{SNR}_{k, (i)}^U = \frac{\mathbb{E}[||\bm{w}_{k,0}^{(i)}- \bm{w}_{k,E}^{(i)}||^2]}{\mathbb{E}[||\bm{e}_k^{(i)}||^2]}.
\end{equation}
Since we assumed that $\bm{e}_k^{(i)}$ is a  zero mean
noise, the variance can be depicted as: 
\begin{equation}
    \label{eq:var-uplink}
    \bm{U}^2_{k,i} := \mathbb{E}\Big[\Big\|\bm{e}^{i}_k\Big\|^2\Big].
\end{equation}
Finally, the weights received from all the participating clients are aggregated and
formulated in line 12 in Algorithm \ref{alg:noisy-fed-avg} as
\begin{equation}
    \label{eq:noise-model-aggregation}
    \bm{w}_{k+1} = \bm{w}_k -  \frac{1}{r}\sum_{i \in \mathcal{S}_k}(\bm{w}_{k,0}^{(i)} - \bm{w}_{k,E}^{(i)} + \bm{e}^{(i)}_k),
\end{equation}
and the process continues again for all the communication rounds.
\subsection{Main assumptions}
Before we start the analysis, the following is the set of assumptions that we make. Assumptions \ref{as1}, \ref{as-may15}, and \ref{as-sfo} are standard used in analyzing FL setting \cite{das2022faster,karimireddy2020scaffold}. Assumption \ref{as-july27} is referred to as \textbf{Noise model} is also used in \cite{wei2022federated}.
\begin{assumption}[\textbf{Smoothness}]
\label{as1} 
$\ell(\bm{x},\bm{w})$ is $L$-smooth with respect to $\bm{w}$, for all $\bm{x}$. Thus, each $f_i(\bm{w})$ ($i \in [n]$) is $L$-smooth, and so is $f(\bm{w})$.
\begin{equation*}
    \| \nabla f_{i}(\bm{w_1}) - \nabla f_{i}(\bm{w_2})\| \leq L\| \bm{w_1} - \bm{w_2}\|; \text{for any  }i, \bm{w_1}, \bm{w_2}.
\end{equation*}
\end{assumption}
\Cref{as1} is a commonly used assumption in convergence analysis of gradient descent based algorithms, \cite{lian2017can,wang2021cooperative} and it restricts the sudden change in gradients.
\begin{assumption}[\textbf{Non-negativity}]
\label{as-may15}
Each $f_i(\bm{w})$ is non-negative and therefore, $f_i^{*} \triangleq \min f_i(\bm{w}) \geq 0$.
\end{assumption}
\Cref{as-may15} is a standard assumption made and is satisfied by most of the loss functions used in practice. However, if in case a loss function is negative the assumption can be achieved by simply adding a constant offset.
\begin{assumption}[\textbf{Bounded Variance}]
\label{as-sfo}
The variance of the stochastic gradient for each client $i$ is bounded: $\mathbb{E}[||\widetilde{\nabla} f_i(\bm{w}_{k,\tau}^{(i)};\mathcal{B}_{k,\tau}^{(i)}) - \nabla f_i(\bm{w}_{k,\tau}^{(i)})||^2] \leq \sigma^2,\ \forall i = 1, \dots, n$, where $\mathcal{B}_{k,\tau}^{(i)}$ represents the random batch of samples in client $i$ for $\tau^{th}$ local iteration.
\end{assumption}
The \cref{as-sfo} is commonly used in analyzing the convergence of gradient descent-based algorithms, as seen in various works such as \cite{yu2019parallel,li2019convergence,nguyen2018sgd,wei2022federated}. However, some other works have used a stricter assumption that assumes uniformly bounded stochastic gradients, i.e., $\mathbb{E}[||\widetilde{\nabla} f_i(\bm{w}_{k,\tau}^{(i)};\mathcal{B}_{k,\tau}^{(i)}) ||^2] \leq \sigma^2$. This assumption is stronger than \cref{as-sfo} and also does not hold true for convex loss functions as shown in \cite{nguyen2018sgd}.
\begin{assumption}[\textbf{Noise model}]
\label{as-july27}
Both the downlink and uplink noise are independent and have zero mean i.e., $\mathbb{E}[\bm{\nu}_k^{(i)}] = 0$ and $\mathbb{E}[\bm{e}_{k}^i] = 0$ and have a bounded variance i.e., $\mathbb{E}[||\bm{e}_{k}^i||^2]=\bm{U}_{k}^2< \infty$ and $\mathbb{E}[||\bm{\nu}_k^{(i)}||^2]= \bm{N}_{k}^2<\infty$.
\end{assumption}
In a noisy communication scenario, the modelling of noise as an additive white Gaussian noise (AWGN) is extremely common, refer \cite{amiri2020federated,ang2020robust}. Here, \cref{as-july27} provides a weaker notion of  AWGN and makes the problem more general and diverse.
\section{Noisy-FedAvg: Improved Analysis}\label{sec:NF-Analysis}
In this section, we describe the improved convergence analysis of the proposed algorithm.
In addition to addressing complications arising from the simultaneous presence of both uplink and downlink noises, our analysis in this section is done without assuming \textit{Bounded Client Dissimilarity} (BCD) that aims to  limit the extent of client heterogeneity and is a frequently-used assumption in FL theory; see, e.g. \cite{wei2022federated,li2019convergence,karimireddy2020scaffold} is the BCD assumption, i.e.,
\begin{equation}
    \label{eq:bcd}
    \|\nabla f_i(\bm{w}) - \nabla f(\bm{w})\|^2 \leq G^2 \text{ } \forall \text{ } \bm{w}, i \in [n],
\end{equation}
where $G$ is a large constant.  
\begin{algorithm}[t]
	\caption{\texttt{Noisy-FedAvg} %\cite{mcmahan2017communication}
	}
	\label{alg:noisy-fed-avg}
	%\algsetup{indent=2em}
	\begin{algorithmic}[1]
		\STATE {\bfseries Input:} 
		Initial point $\bm{w}_0$, \# of communication rounds $K$, period $E$, learning rates  $\{\eta_{k}\}_{k=0}^{K-1}$, and global batch size $r$.
		%\vspace{0.1 cm}
		\FOR{$k =0,\dots, K-1$}
		%\vspace{0.1 cm}
		\STATE %Server chooses a set $\mathcal{S}_k$ of $r$ clients uniformly at random w/o replacement \& sends them $\bm{w}_k$. %to them.
		Server sends $\bm{w}_k$ to a set $\mathcal{S}_k$ of $r$ clients chosen uniformly at random without replacement.
		%\vspace{0.1 cm}
		\FOR{client $i \in \mathcal{S}_k$}
		%\vspace{0.1 cm}
		\STATE {\bfseries Downlink communication:}
		Broadcasting $\bm{w}_k$ through a noisy downlink communication channel having zero mean. Set $\bm{w}_{k,0}^{(i)} = \bm{w}_k + \bm{\nu}_k^{(i)}$, where $\bm{\nu}_k^{(i)}$ is the downlink noise.
		%\vspace{0.1 cm}
		\FOR{$\tau = 0,\ldots,E-1$}
		%\vspace{0.1 cm}
		\STATE Pick a random batch of samples 
		in client $i$, $\mathcal{B}_{k,\tau}^{(i)}$.
		Compute the stochastic gradient of $f_i$ at $\bm{w}_{k,\tau}^{(i)}$ over %a random batch of samples in client $i$, 
		$\mathcal{B}_{k,\tau}^{(i)}$, viz. $\widetilde{\nabla} f_i(\bm{w}_{k,\tau}^{(i)};\mathcal{B}_{k,\tau}^{(i)})$.
		%\vspace{0.1 cm}
		\STATE Update $\bm{w}_{k,\tau+1}^{(i)} = \bm{w}_{k,\tau}^{(i)} - \eta_{k} \widetilde{\nabla} f_i(\bm{w}_{k,\tau}^{(i)};\mathcal{B}_{k,\tau}^{(i)})$.
		\vspace{0.1 cm}
		\ENDFOR
% 		%\vspace{0.1 cm}
% 		\STATE Send uplink noise, $\bm{e}^{(i)}_k$.
% 		\label{line:fedavg-1_1}
% 		%\vspace{0.1 cm}
		%\vspace{0.1 cm}
		\STATE {\bfseries Uplink communication:}
		$(\bm{w}_k - \bm{w}_{k,E}^{(i)} )$ goes to the server through a noisy uplink communication channel having zero mean. So, send $(\bm{w}_{k,0}^{(i)} - \bm{w}_{k,E}^{(i)} +\bm{e}_{k}^{(i)})$, where $\bm{e}^{(i)}_k$ is the uplink noise.
		\label{line:fedavg-1}
		%\vspace{0.1 cm}
		\ENDFOR
		%\vspace{0.1 cm}
		\STATE Update $\bm{w}_{k+1} = \bm{w}_k -  \frac{1}{r}\sum_{i \in \mathcal{S}_k}(\bm{w}_{k,0}^{(i)} - \bm{w}_{k,E}^{(i)} + \bm{e}^{(i)}_k)$.
		\label{line:fedavg-2}
		%\vspace{0.1 cm}
		\ENDFOR
	\end{algorithmic}
\end{algorithm}
Furthermore, in contrast to  \cite{wei2022federated,li2019convergence}, we do not make any assumption about the \textit{strong convexity} of the loss function.

\subsection{Noisy-SGD}
To give more insight into the analysis of \texttt{Noisy-FedAvg} and its implications we first consider a fictitious scenario where a noisy version of  SGD is employed to minimize a stochastic, non-convex, and $L$-smooth function with the following update
\begin{equation}
    \label{eq:noisy-SGD}
    \bm{w}_{t+1} = \bm{w}_{t} - \eta[\bm{e}_{t} + \widetilde\nabla {f}( \bm{w}_{t} + \bm{\nu}_{t}; \mathcal{B}_t)],
\end{equation}
where $\bm{e}_{t}$ and $\bm{\nu}_{t}$ can be thought of as uplink and downlink noise, respectively.
The purpose of this analysis is to shed light on the effect of noise on SGD-based FL algorithms.
\begin{theorem}[\textbf{Smooth non-convex case for \texttt{Noisy-SGD}}]
\label{thm-noisy-SGD}
Let $f:\mathbb{R}^{d} \rightarrow \mathbb{R}$ be a $L$-smooth non-convex function and $f^{*} := \inf_{w \in \mathbb{R}^{d}} f(w)$. Consider the noisy-SGD method with the update in \cref{eq:noisy-SGD}. Let 
%$\bm{U}^2_{t}$ and $\bm{N}^{2}_{t}$ be the maximum uplink and downlink variance for the $t^{th}$ iteration 
$\bm{e}_{t}$ and $\bm{\nu}_{t}$ satisfy \Cref{as-july27} and the stochastic gradient satisfy \Cref{as-sfo}. If, $\eta_{t} = \eta$ and $\eta \leq \frac{1}{L}$, then for all $t \in \{0,\ldots,T-1\}$ noisy-SGD satisfies
\begin{align}
\label{eq: thm-noisy-SGD}
 \frac{1}{T}\sum_{t = 0}^{T-1}\E_{\{\mathcal{B}_t,\bm{e}_{t},\bm{\nu}_{t}\}_{t=0}^{T-1}}[||\nabla f(\bm{w}_{t}) ||^2]] \leq \frac{2\big(f(\bm{w}_{0}) - f^{*}\big)}{T \eta}  + \eta L\sigma^2 + \underbrace{{\color{black}\frac{L^2}{T}\sum_{t = 0}^{T-1}\bm{N}_{t}^{2}}}_{\text{Term I}} + \underbrace{{\color{black}\frac{\eta L}{T}\sum_{t = 0}^{ T-1}\bm{U}_{t}^{2}}}_{\text{Term II}}.
\end{align}
\end{theorem}
\begin{proof}
% Similar to Assumption \ref{as-july27} (Noise model), we assume the $\bm{e}_{t}$ and $\bm{\nu}_{t}$ are independent and have zero mean, i.e., $\mathbb{E}[\bm{e}_{t}] = 0$ and $\mathbb{E}[\bm{\nu}_{t}] = 0$ and have a bounded variance i.e., $\mathbb{E}[||\bm{e}_{t}||^2]\leq \bm{U}_{t}^2$ and $\mathbb{E}[||\bm{\nu}_{t}||^2]\leq \bm{N}_{t}^2$. 
% %Additionally, we have $\mathbb{E}[\widetilde\nabla {f}(x)] = \nabla f(x)$ and $\mathbb{E}[||\widetilde\nabla {f}(x)] - \nabla f(x)||^2]\leq \sigma^2$ which follow the standard set template of SGD. 
% We analyze this under non-convexity and 
Using $L$-smoothness assumption we can obtain,
\begin{align}
\label{eq: noisy-SGD-smooth}
f(\bm{w}_{t+1}) \leq f(\bm{w}_{t}) + \langle \nabla f(\bm{w}_{t}), \bm{w}_{t+1} - \bm{w}_{t} \rangle
+ \frac{L}{2}||\bm{w}_{t+1} - \bm{w}_{t}||^2.
\end{align}
\normalsize
Using \cref{eq:noisy-SGD} in \cref{eq: noisy-SGD-smooth} yields
\begin{align}
\label{eq: noisy-SGD-smooth+noise}
f(\bm{w}_{t+1}) \leq f(\bm{w}_{t}) - \underbrace{\eta \langle \nabla f(\bm{w}_{t}), \bm{e}_{t} + \widetilde\nabla {f}( \bm{w}_{t} + \bm{\nu}_{t}; \mathcal{B}_t) \rangle}_{(A)} +\underbrace{\frac{\eta^2 L}{2}||\bm{e}_{t} + \widetilde\nabla {f}( \bm{w}_{t} + \bm{\nu}_{t}; \mathcal{B}_t)||^2}_{(B)}. 
\end{align}
Using $A:$
\begin{align}
\nonumber
    A &= \eta \langle \nabla f(\bm{w}_{t}), \bm{e}_{t} + \widetilde\nabla {f}( \bm{w}_{t} + \bm{\nu}_{t}; \mathcal{B}_t) \rangle
    \\
    \label{eq:nsgd-A}
    & = \eta \langle \nabla f(\bm{w}_{t}), \bm{e}_{t}\rangle + \eta \langle \nabla f(\bm{w}_{t}), \widetilde\nabla {f}( \bm{w}_{t} + \bm{\nu}_{t}; \mathcal{B}_t) \rangle
\end{align}
In \cref{eq:nsgd-A}, taking expectation with respect to $\bm{e}_{t}$, will result in \cref{eq:nsgd-A-final}.
\begin{align}
  \label{eq:nsgd-A-final}
    \E_{\bm{e}_{t}}[A] & = \eta \langle \nabla f(\bm{w}_{t}), \widetilde\nabla {f}( \bm{w}_{t} + \bm{\nu}_{t}; \mathcal{B}_t) \rangle
\end{align}
Using $B:$
\begin{align}
    \nonumber
    B = \frac{\eta^2 L}{2}||\bm{e}_{t} + \widetilde\nabla {f}( \bm{w}_{t} + \bm{\nu}_{t}; \mathcal{B}_t) - \nabla {f}( \bm{w}_{t} + \bm{\nu}_{t}) 
    + \nabla {f}( \bm{w}_{t} + \bm{\nu}_{t})||^2
    \\
    \nonumber
    = \frac{\eta^2 L}{2}||\bm{e}_{t}||^2 + \frac{\eta^2 L}{2}||\widetilde\nabla {f}( \bm{w}_{t} + \bm{\nu}_{t}; \mathcal{B}_t) - \nabla {f}( \bm{w}_{t} + \bm{\nu}_{t})||^2 
    + \frac{\eta^2 L}{2}||\nabla {f}( \bm{w}_{t} + \bm{\nu}_{t})||^2 + 2\langle \bm{e}_{t},\nabla {f}( \bm{w}_{t} + \bm{\nu}_{t})\rangle
    \\
    \nonumber
    + 2\langle \bm{e}_{t}, \widetilde\nabla {f}( \bm{w}_{t} + \bm{\nu}_{t}; \mathcal{B}_t) - \nabla {f}( \bm{w}_{t} + \bm{\nu}_{t})\rangle
    \\
  \label{eq:nsgd-B}
    + 2\langle \widetilde\nabla {f}( \bm{w}_{t} + \bm{\nu}_{t}; \mathcal{B}_t) - \nabla {f}( \bm{w}_{t} + \bm{\nu}_{t}),  \nabla {f}( \bm{w}_{t} + \bm{\nu}_{t})\rangle 
\end{align}
Taking the expectation with respect to $\bm{e}_t$ which has zero mean, alongside the independence assumption of noises in \cref{eq:nsgd-B}, we can re-write $B$ as
\begin{align}
    \nonumber
    \E_{\bm{e}_{t}}[B] =  \frac{\eta^2 L}{2}||\widetilde\nabla {f}( \bm{w}_{t} + \bm{\nu}_{t}; \mathcal{B}_t) - \nabla {f}( \bm{w}_{t} + \bm{\nu}_{t})||^2 
    + 2\langle \widetilde\nabla {f}( \bm{w}_{t} + \bm{\nu}_{t}; \mathcal{B}_t)    - \nabla {f}( \bm{w}_{t} + \bm{\nu}_{t}),  
    \nabla {f}( \bm{w}_{t} + \bm{\nu}_{t})\rangle 
    \\
    \label{eq:nsgd-B-final}
    + \frac{\eta^2 L}{2} \bm{U}_{t}^{2} + \frac{\eta^2 L}{2}||\nabla {f}( \bm{w}_{t} + \bm{\nu}_{t})||^2
\end{align}
Now, putting the results of \cref{eq:nsgd-A-final,eq:nsgd-B-final} in \cref{eq: noisy-SGD-smooth+noise} along with taking the expectation with respect to data, we get
\begin{multline}
\label{eq: noisy-SGD-smooth+expec}
\E_{\mathcal{B}_t,\bm{e}_{t}}[f(\bm{w}_{t+1})] \leq  \E_{\mathcal{B}_t,\bm{e}_{t}}[f(\bm{w}_{t})] +\frac{\eta^2 L}{2}\bm{U}_{t}^2 + \frac{\eta^2 L}{2} \sigma^2 
- \eta \E_{\mathcal{B}_t,\bm{e}_{t}}[\langle \nabla f(\bm{w}_{t}),\nabla {f}( \bm{w}_{t} + \bm{\nu}_{t}) \rangle] \\+ \frac{\eta^2 L}{2}\E_{\mathcal{B}_t,\bm{e}_{t}}[||\nabla {f}( \bm{w}_{t} + \bm{\nu}_{t})||^2 ].
\end{multline}
For any 2 vectors $\bm{a}$ and $\bm{b}$, we have that
\begin{flalign}
    \label{eq:feb28-103}
   - \langle \bm{a}, \bm{b} \rangle = \frac{1}{2}(\|\bm{a} - \bm{b}\|^2-\|\bm{a}\|^2 - \|\bm{b}\|^2).
\end{flalign}
Using this in \cref{eq: noisy-SGD-smooth+expec} we get
\begin{multline}
\label{eq: noisy-SGD-smooth+trickvector}
\E_{\mathcal{B}_t,\bm{e}_{t}}[f(\bm{w}_{t+1})] \leq  \E_{\mathcal{B}_t,\bm{e}_{t}}[f(\bm{w}_{t})] + \frac{\eta^2 L}{2}(\bm{U}_{t}^{2}+\sigma^2)
+ \frac{\eta}{2} \E_{\mathcal{B}_t,\bm{e}_{t}}\Big[||\nabla f(\bm{w}_{t}) -\nabla {f}( \bm{w}_{t} + \bm{\nu}_{t})||^2 -||\nabla {f}( \bm{w}_{t} + \bm{\nu}_{t})||^2 -||\nabla f(\bm{w}_{t}) ||^2\Big]  \\+\frac{\eta^2 L}{2}\E_{\mathcal{B}_t,\bm{e}_{t}}[||\nabla {f}( \bm{w}_{t} + \bm{\nu}_{t})||^2 ].
\end{multline}
If $\eta \leq \frac{1}{L}$, we can drop $\E_{\mathcal{B}_t,\bm{e}_{t}}[||\nabla {f}( \bm{w}_{t} + \bm{\nu}_{t})||^2$ as it will appear with a negative sign in the RHS of  \cref{eq: noisy-SGD-smooth+trickvector}. Consequently, using $L$-smoothness yields
\begin{align}
\label{eq: noisy-SGD-smooth+eta}
\E_{\mathcal{B}_t,\bm{e}_{t}}[f(\bm{w}_{t+1})] \leq  \E_{\mathcal{B}_t,\bm{e}_{t}}[f(\bm{w}_{t})] +\frac{\eta^2 L}{2}(\bm{U}_{t}^{2}+\sigma^2) + \frac{\eta}{2} \E_{\mathcal{B}_t,\bm{e}_{t}}\Big[L^2||(\bm{\nu}_{t})||^2 - ||\nabla f(\bm{w}_{t}) ||^2\Big].
\end{align}
Taking expectation w.r.t. $\bm{\nu}_{t}$ we have
\begin{align}
\label{eq: noisy-SGD-smooth+expec+nu}
\E_{\mathcal{B}_t,\bm{e}_{t},\bm{\nu}_{t}}[f(\bm{w}_{t+1})] \leq  \E_{\mathcal{B}_t,\bm{e}_{t},\bm{\nu}_{t}}[f(\bm{w}_{t})] +\frac{\eta}{2}L^2 \bm{N}_{t}^2 -\frac{\eta}{2} \E_{\mathcal{B}_t,\bm{e}_{t},\bm{\nu}_{t}}[||\nabla f(\bm{w}_{t}) ||^2] + \frac{\eta^2 L}{2}(\bm{U}_{t}^{2}+\sigma^2).
\end{align}
Summing the above equation for $t= 0, 1, \dots, T-1$ and dividing both sides by $T\eta/2$
we obtain,
\begin{align}
\label{eq: noisy-SGD-smooth+expec+nu-final}
\frac{1}{T}\sum_{t = 0}^{T-1} \E_{\{\mathcal{B}_t,\bm{e}_{t},\bm{\nu}_{t}\}_{t=0}^{T-1}}[||\nabla f(\bm{w}_{t}) ||^2] \leq  \frac{2\big(f(\bm{w}_{0}) - f(\bm{w}_{T})\big)}{T \eta} + \eta L \sigma^2 
+ \frac{L^2}{T} \sum_{t=0}^{T-1} \bm{N}_{t}^2 + \frac{\eta L}{T} \sum_{t=0}^{T-1}\bm{U}_{t}^{2}.
\end{align}
Now using the fact that $f^* \leq f(\bm{w}_{T})$ in the equation above, we obtain the stated result in \cref{eq: thm-noisy-SGD}.
% \begin{multline}
% \label{eq: final-noisy-SGD}
%     \frac{1}{T}\sum_{t = 0}^{t = T-1}\mathbb{E}[||\nabla f(\bm{w}_{t}) ||^2]] \leq \frac{2(f(\bm{w}_{0}) - f(\bm{w}^{*}))}{T \eta} + \eta L\sigma^2 \\ + {\color{black}\frac{L^2}{T}\sum_{t = 0}^{t = T-1}\bm{N}_{t}^{2}} + {\color{black}\frac{\eta L}{T}\sum_{t = 0}^{t = T-1}\bm{U}_{t}^{2}}
% \end{multline}
% This concludes the proof for \texttt{Noisy-SGD}.
\end{proof}
The implication of \cref{eq: thm-noisy-SGD} is that the downlink noise (Term I) is more degrading than the uplink noise (Term II) given that the effect of the latter on the convergence can be controlled by $\eta$. That is, uplink noise slows the convergence rate while the downlink noise may inhibit the convergence. The following corollary describes a SNR control strategy that aims to recover the rate of noise-free SGD by pushing the noise-driven terms, i.e., Terms I and II in the RHS of \cref{eq: thm-noisy-SGD}, to the higher-order term. In the context of this paper, a higher-order term deviant from the dominant term is one that does not control the order of convergence error.
\begin{corollary}
If $\eta \leq \frac{1}{L}$ and $\eta  = \mathcal{O}(\frac{1}{\sqrt{T}})$ and a SNR control strategy is employed such that $\frac{1}{T}\sum_{t = 0}^{T-1}\bm{U}_{t}^{2} = \mathcal{O}(T^{-\delta_1})$ and $\frac{1}{T}\sum_{t = 0}^{ T-1}\bm{N}_{t}^{2} = \mathcal{O}(\frac{1}{T^{0.5+\delta_2}})$ for some $\delta_1,\delta_2>0$, the dominant term in the convergence error of \texttt{Noisy-SGD} will be $\mathcal{O}(\frac{1}{\sqrt{T}})$ which is independent of noise characteristics and hence similar to the noise-free case of SGD. 
\end{corollary}

\subsection{Noisy-FedAvg}
In what follows, we build upon \Cref{thm-noisy-SGD} to present Theorem \ref{thm-noisy-fedavg}, which holds for both partial and full client participation, IID, and non-IID data distribution. 
\begin{theorem}[\textbf{Smooth non-convex case for \texttt{Noisy-FedAvg}}]
\label{thm-noisy-fedavg}
Let Assumptions \ref{as1}, \ref{as-may15}, \ref{as-sfo}, \ref{as-july27} holds for \texttt{Noisy-FedAvg} (\Cref{alg:noisy-fed-avg}). 
% Further, suppose Assumption \ref{as1} holds for \texttt{FedAvg} (\Cref{alg:fed-avg}). 
%Let $\sigma^2$ be the maximum variance of the local (client-level) stochastic gradients. Also, let $\bm{E}_{k}^2$ and $\bm{N}_{k}^2$ be the maximum uplink and downlink variance respectively for the $k$th communication round.
In \texttt{Noisy-FedAvg}, set $\eta_{k} = \frac{1}{\gamma L E}\sqrt{\frac{r}{K}}$ for all $k$, where $\gamma > 4$ is a universal constant. 
Define a distribution $\mathbb{P}$ for $k \in \{0,\ldots,K-1\}$ such that $\mathbb{P}(k) = \frac{(1+\zeta)^{(K-1-k)}}{\sum_{k=0}^{K-1}(1+\zeta)^k}$ where $\zeta := 8\eta^2 L^2 E^2 \Big(\frac{(n-r)}{r(n-1)} + {\frac{2 \eta L E}{3}}\Big)$. Sample $k^{*}$ from $\mathbb{P}$ uniformly. Then, for $K \geq \max\Big(\frac{1024 r^3}{9\gamma^2}(\frac{1}{\gamma^2-16})^2, \frac{4r}{\gamma^2}\Big)$,
\begin{multline}
    \E[\|\nabla f(\bm{w}_{k^{*}})\|^2] \leq \frac{8 \gamma L f(\bm{w}_0)}{\sqrt{r K}} + \underbrace{{\color{black}\frac{4}{\gamma E^2 K \sqrt{r K}}}  {\color{black}\sum_{k=0}^{K-1}\bm{U}^2_k}}_{\text{Term I}}
    + \underbrace{\frac{4}{\gamma E} \sqrt{\frac{r}{K}}\Big(\frac{1}{\gamma n} \sqrt{\frac{r}{K}}(1 + 
    \frac{2n E}{3} + n) + \frac{1}{r} + \frac{(n-r)}{r(n-1)}\Big)\sigma^2}_{\text{Term II}}
    \\  \underbrace{{+\color{black}\frac{{4L^2}}{E K}} {\color{black}\Big(1 + 4 E + \frac{2}{\gamma E}\sqrt{\frac{r}{K}}(1} {\color{black}+2 E^2\{\frac{3}{\gamma E^2}\sqrt{\frac{r}{K}} + 2(2}
    {\color{black}+\frac{3}{\gamma^2 E^2}\frac{r}{K})}
    {\color{black}(\frac{2}{3\gamma}\sqrt{\frac{r}{K}}} {\color{black}+\frac{(n-r)}{r(n-1)})\})\Big)} {\color{black}\sum_{k=0}^{K-1}\bm{N}^2_k.}
    }_{\text{Term III}}
\end{multline}
\end{theorem}
In the \Cref{thm-noisy-fedavg} the expectation is w.r.t the choice of clients, data, uplink, and downlink noises. The Terms I and III in the theorem above depict the effects of uplink and downlink noises respectively. Furthermore, Term II is a direct consequence of \Cref{as-sfo}, stemming from the stochastic gradients of the clients. Now, from \Cref{thm-noisy-fedavg}, mirroring the result of \Cref{thm-noisy-SGD}, we can observe that the uplink noise is not dominant compared to the downlink noise. As we will discuss in \Cref{section:exp}, our tight analysis in establishing \ref{thm-noisy-fedavg} is verified numerically as well by showing that uplink noise's impact is not as detrimental as downlink noise.

\begin{remark}
    \Cref{thm-noisy-fedavg} is derived without using the restrictive assumption of BCD, in \cref{eq:bcd}. Now, to clarify why this assumption is restrictive, let us consider a toy example of a univariate quadratic function. The following notations hold their usual meaning as defined in \cref{sec:prelim}.
\begin{gather}
    f_{i}(\bm{w}) = \frac{1}{2}(\bm{w}^2), \forall i = 1, \dots, n-1
    \\
    \label{eq:f_n}
    f_{n}(\bm{w}) = \bm{w}^2
\end{gather}
By using \cref{eq:objective-function}, a global objective function can be formulated as,
\begin{align}
\label{eq:bcd-f}
    f(\bm{w}) = \frac{1}{2n}(n+1)\bm{w}^2
\end{align}
For $i = n$, by taking the gradient of \cref{eq:f_n,eq:bcd-f}, and putting them back in \cref{eq:bcd}, we get
\begin{gather}
\nonumber
    \|\nabla f_i(\bm{w}) - \nabla f(\bm{w})\|^2 \leq G^2
    \\
    \nonumber
    \implies 
    \|2\bm{w} - \frac{n+1}{n}\bm{w}\|^2 \leq G^2, \text{where  } i=n 
    \\
    \label{eq:bcd-eg}
    \implies
    \|\bm{w}\|^2 \|1 - \frac{1}{n}\|^2 \leq G^2
\end{gather}
We can infer from \cref{eq:bcd-eg} that the inequality does not hold for every $\bm{w}\in \mathbb{R}$, given a fixed $G$ and hence it makes the assumption restrictive and somewhat unrealistic.
\end{remark}

\begin{remark}
Before we start with the proof of the \Cref{thm-noisy-fedavg}, we would like to emphasize that the theorem provides an upper bound on the performance for the FedAvg algorithm, which depends on the noise characteristics. It is an interesting future work to investigate the effect of noise on the lower bound that essentially bounds the performance of any algorithm in this scenario and see if one could use such lower bounds towards SNR scaling (for more details on SNR scaling refer \cref{sec:snr}).
\end{remark}
\begin{proof}
The proof is motivated by the approach taken in \cite{das2022faster}. However, with the inclusion of downlink and uplink noises, the first local update (refer \cref{eq:noisy-global-update} and the model update (refer \cref{eq:noise-model-update}) are considerably different  which makes the analysis significantly different and more involved. As will be outlined shortly, the proof relies on careful treatment of the first local and model updates using new techniques.

To commence the proof, using \Cref{sep26-lem3}, for $\eta_k L E \leq \frac{1}{2}$, we can bound the per-round progress as:
\begin{flalign}
    \nonumber
    \mathbb{E}[f(\bm{w}_{k+1})] 
    \leq \mathbb{E}[f(\bm{w}_k)] - \frac{\eta_k (E-1)}{2} \mathbb{E}[\|\nabla f(\bm{w}_k)\|^2]
    + 4\eta_k^2 L E^2 \Big(\frac{(n-r)}{r(n-1)} + \frac{2}{3}\eta_k L E\Big)\Big(\frac{1}{n}\sum_{i \in [n]} \mathbb{E}[\|\nabla {f}_i(\bm{w}_{k})\|^2]\Big) 
    \\ \nonumber
    + \eta_k^2 L E \Big(\frac{\eta_k L E}{n}\Big(1 + \frac{2nE}{3} + n\Big) + \frac{1}{r} + \frac{(n-r)}{r(n-1)}\Big)\sigma^2
    + \frac{\eta^2_kL}{2r}\frac{1}{n}\sum_{i \in [n]}\bm{U}^2_{k,i}
    + \frac{\eta_kL^2}{2}\Big(1 + 2\eta_kL + 4E\{1+3\eta^2_kL^2\\
     \label{eq:sept26-20}
    + 2\eta_kLE(2+3\eta^2_kL^2)(\frac{2}{3}\eta_kLE +\frac{(n-r)}{r(n-1)})\}\Big)\frac{1}{n}\sum_{i \in [n]}\bm{N}^2_{k,i}.
\end{flalign}
Now using the $L$-smoothness and non-negativity of the $f_i$'s, we get:
\begin{flalign*}
    \sum_{i \in [n]} \mathbb{E}[\|\nabla f_i(\bm{w}_k)\|^2] & \leq \sum_{i \in [n]} 2L(\mathbb{E}[f_i(\bm{w}_k)] - f_i^{*}) 
    \\
    & \leq 2n L \mathbb{E}[f(\bm{w}_k)] - 2L \sum_{i \in [n]} f_i^{*} 
    \\
    &\leq 2n L \mathbb{E}[f(\bm{w}_k)]
\end{flalign*}
Putting this result in \cref{eq:sept26-20}, we get for a constant learning rate of $\eta_k = \eta,   \bm{U}^2_{k,i} = \bm{U}^2_k$ and $\bm{N}^2_{k,i} = \bm{N}^2_k$:
\begin{multline*}
    % \label{eq:sept26-18-1}
    \mathbb{E}[f(\bm{w}_{k+1})]
    \leq 
    \Big(1 + 8\eta^2 L^2 E^2 \Big(\frac{(n-r)}{r(n-1)} + {\frac{2 \eta L E}{3 }}\Big)\Big) \mathbb{E}[f(\bm{w}_k)]
     - \frac{\eta (E-1)}{2} \mathbb{E}[\|\nabla f(\bm{w}_k)\|^2] + \eta^2 L E \Big(\frac{\eta L E}{n}(1 
    + \frac{2nE}{3} \\ + n)
    + \frac{1}{r} + \frac{(n-r)}{r(n-1)}\Big)\sigma^2 + \frac{\eta^2L}{2r}\frac{1}{n}\sum_{i \in [n]}\bm{U}^2_k 
    + 
    \frac{\eta L^2}{2}\Big(1 + 2\eta L + 4E\{1+
    3\eta^2L^2 + 2\eta LE(2+3\eta^2L^2)(\frac{2\eta LE}{3} +\frac{(n-r)}{r(n-1)})\}\Big)
    \frac{1}{n}\sum_{i \in [n]}\bm{N}^2_k.
\end{multline*}
For ease of notation, define $\zeta := 8\eta^2 L^2 E^2 \Big(\frac{(n-r)}{r(n-1)} + {\frac{2 \eta L E}{3}}\Big)$, $\zeta_2 := \Big(\frac{\eta L E}{n}\Big(1 + \frac{2n E}{3} + n\Big) + \frac{1}{r} + \frac{(n-r)}{r(n-1)}\Big)$ and $\zeta_3 := \Big(1 + 2\eta L + 4E\{1+3\eta^2 L^2 + 2\eta L E(2+3\eta^2 L^2)(\frac{2\eta L E}{3}+\frac{(n-r)}{r(n-1)})\}\Big) $. Then, unfolding the recursion of the equation above from  $k=0$ through to $k=K-1$, we get:
\begin{multline}
    \label{eq:sept26-19}
    \mathbb{E}[f(\bm{w}_{K})] 
    \leq (1 + \zeta)^K f(\bm{w}_0)
    - \frac{\eta (E-1)}{2} \sum_{k=0}^{K-1}(1+\zeta)^{(K-1-k)} \mathbb{E}[\|\nabla f(\bm{w}_k)\|^2]
    + \eta^2 L E \zeta_2 \sigma^2 \sum_{k=0}^{K-1}(1+\zeta)^{(K-1-k)}
    \\
    + \frac{\eta^2 L}{2r} \sum_{k=0}^{K-1} \bm{U}^2_k(1+\zeta)^{(K-1-k)}
    + \frac{\eta L^2}{2} \zeta_3 \sum_{k=0}^{K-1}\bm{N}^2_k(1+\zeta)^{(K-1-k)}.
\end{multline}

Let us define $p_k := \frac{(1+\zeta)^{(K-1-k)}}{\sum_{k'=0}^{K-1}(1+\zeta)^{(K-1-k')}}$. Then, re-arranging \cref{eq:sept26-19} and using the fact that $\mathbb{E}[f(\bm{w}_{K})] \geq 0$ and $\frac{\eta(E-1)}{2} > \frac{\eta E}{4}$, we get:
\begin{flalign}
    \nonumber
    \sum_{k=0}^{K-1}p_k \mathbb{E}[\|\nabla f(\bm{w}_k)\|^2]  \leq \frac{4 (1 + \zeta)^K f(\bm{w}_0)}{\eta E \sum_{k'=0}^{K-1}(1+\zeta)^{k'}} + {4\eta L}\zeta_2 \sigma^2 
    + \frac{2\eta L}{rE} \frac{\sum_{k=0}^{K-1}\bm{U}^2_k(1+\zeta)^{(K-1-k)}}{\sum_{k'=0}^{K-1}(1+\zeta)^{k'}}
    \\ 
     + \frac{2L^2\zeta_3}{E} \frac{\sum_{k=0}^{K-1}\bm{N}^2_k(1+\zeta)^{(K-1-k)}}{\sum_{k'=0}^{K-1}(1+\zeta)^{k'}}
    \\
    \label{eq:sep27-1}
     = \frac{4 \zeta f(\bm{w}_0)}{\eta E
    (1 - (1+\zeta)^{-K})} + 4\eta L \zeta_2 \sigma^2
    +\frac{2\eta L}{rE} \frac{\zeta \sum_{k=0}^{K-1}\bm{U}^2_k}{(1+\zeta) - (1+\zeta)^{-K+1}}
     + \frac{2L^2 \zeta_3}{E} \frac{\zeta \sum_{k=0}^{K-1}\bm{N}^2_k}{(1+\zeta) - (1+\zeta)^{-K+1}}
\end{flalign}

where the \cref{eq:sep27-1} follows by using the fact that $\sum_{k'=0}^{K-1}(1+\zeta)^{k'} = \frac{(1+\zeta)^{K} - 1}{\zeta}$ and H\"{o}lder's Inequality. Now,

\begin{flalign}
    \nonumber
    (1+\zeta)^{-K} & < 1 - \zeta K + {\zeta^2}\frac{K(K+1)}{2} < 1 - \zeta K + {\zeta^2}K^2 
    \\ \label{eq:trick-1}
    & \implies 1 - (1+\zeta)^{-K} > \zeta K (1 - \zeta K).
\end{flalign}

Also,
\begin{flalign}
    \nonumber
    (1+\zeta)^{-K+1} & < 1 + \zeta(-K+1) + {\zeta^2}\frac{(-K)(-K+1)}{2} 
    \\ \nonumber
    & < (1+\zeta) - \zeta K + {\zeta^2}K^2 
    \\ \label{eq:trick-2}
    & \implies (1+\zeta) - (1+\zeta)^{-K+1} > \zeta K (1 - \zeta K).
\end{flalign}

Plugging \cref{eq:trick-1,eq:trick-2} with $\zeta_2$ and $\zeta_3$ in \cref{eq:sep27-1}, we have for $\zeta K < 1$:
\begin{flalign}
\nonumber
    \sum_{k=0}^{K-1}p_k \mathbb{E}[\|\nabla f(\bm{w}_k)\|^2] \leq \frac{4 f(\bm{w}_0)}{\eta E K (1 - \zeta K)} + 4\eta L E \Big(\frac{\eta L}{n}\Big(1
    + \frac{2nE}{3} + n\Big) + \frac{1}{r E} + \frac{(n-r)}{r(n-1)E}\Big)\sigma^2 
    + \frac{2\eta L}{rE} \frac{\sum_{k=0}^{K-1}\bm{U}^2_k}{K(1 - \zeta K)}
    \\ \label{eq:sep27-2}
    + \frac{2L^2}{E K(1 - \zeta K)} \Big(1 + 2\eta L + 4E\{1+3\eta^2 L^2 + 2\eta L E(2 
    +3\eta^2 L^2)
    (\frac{2\eta L E}{3} +\frac{(n-r)}{r(n-1)})\}\Big) \sum_{k=0}^{K-1}\bm{N}^2_k.
\end{flalign}
In this case, note that the optimal step size will be $\eta = \mathcal{O}(\frac{1}{L E \sqrt{K}})$, even for $r = n$. 
So let us pick $\eta = \frac{1}{\gamma L E}\sqrt{\frac{r}{K}}$, where $\gamma$ is some constant such that $\gamma > 4$. Note that we need to have $\eta L E \leq \frac{1}{2}$; this happens for $K \geq \frac{4r}{\gamma^2}$. Further, let us ensure $\zeta K < \frac{1}{2}$; this happens for $K \geq \frac{1024 r^3}{9\gamma^2}(\frac{1}{\gamma^2-16})^2$. Thus, we should have $K \geq \max\Big(\frac{1024 r^3}{9\gamma^2}(\frac{1}{\gamma^2-16})^2, \frac{4r}{\gamma^2}\Big)$. Putting $\eta = \frac{1}{\gamma L E}\sqrt{\frac{r}{K}}$ in \cref{eq:sep27-2} and also using $1 - \zeta K \geq \frac{1}{2}$, we get:
\begin{multline}
    \label{eq:final-1}
    \sum_{k=0}^{K-1}p_k \mathbb{E}[\|\nabla f(\bm{w}_k)\|^2] \leq \frac{8 \gamma L f(\bm{w}_0)}{\sqrt{r K}} + {\color{black}\frac{4}{\gamma E^2 K \sqrt{r K}}}  {\color{black}\sum_{k=0}^{K-1}\bm{U}^2_k}
    +
    \frac{4}{\gamma E} \sqrt{\frac{r}{K}}\Big(\frac{1}{\gamma n} \sqrt{\frac{r}{K}}(1 + 
    \frac{2n E}{3} + n) + \frac{1}{r} + \frac{(n-r)}{r(n-1)}\Big)\sigma^2 
    \\
    +{\color{black}\frac{{4L^2}}{E K}} {\color{black}\Big(1 + 4 E + \frac{2}{\gamma E}\sqrt{\frac{r}{K}}(1 +2 E^2\{\frac{3}{\gamma E^2}\sqrt{\frac{r}{K}} + 2(2+}
    {\color{black}\frac{3}{\gamma^2 E^2}\frac{r}{K})}
    {\color{black}(\frac{2}{3\gamma}\sqrt{\frac{r}{K}} +\frac{(n-r)}{r(n-1)})\})\Big)} {\color{black}\sum_{k=0}^{K-1}\bm{N}^2_k}
    .
\end{multline}
%Using the fact that $\frac{1}{E} + \frac{(n-r)}{3(n-1)} \leq \frac{4}{3}$ gives us the final result.
This finishes the proof.

\end{proof}
\section{Theory guided SNR control}\label{sec:snr}
\Cref{thm-noisy-fedavg} provides us with actionable insight into the effect of uplink and downlink noise on model convergence, i.e., the inherent asymmetry of their adverse effect on the performance of FL algorithms. One strategy to improve the convergence properties of the model in noisy settings is to boost the SNR of the communicated messages (see, e.g.\cite{wei2022federated} and the references therein). In this section, we aim to establish an improved SNR control strategy following the improved analysis presented in \Cref{thm-noisy-fedavg}. First, we stated the following corollary for an easier-to-interpret result.
\begin{corollary}
\label{coro-noisy-fedavg}
Instate the notation and hypotheses of Theorem \ref{thm-noisy-fedavg}. Also, let $\bm{U}^2_k \leq \bm{e}^2$ and $\bm{N}^2_k \leq \bm{\nu}^2$ to be the maximum bounded variances for all $k$. Then, if  $K =\Omega(r^3)$,
%$K \geq \mathcal{O}\Big(\max\Big(\frac{1024 r^3}{9\gamma^2}(\frac{1}{\gamma^2-16})^2, \frac{4r}{\gamma^2}\Big)\Big)$
\begin{align}
    \label{eq:coro-1}
    \mathbb{E}[\|\nabla f(\bm{w}_{k^{*}})\|^2] = \mathcal{O}\Big(\frac{1}{E}\sqrt{\frac{r}{K}} \sigma^2 + {\color{black}\frac{1}{E^2 \sqrt{r K}} \bm{e}^2} +
    {\color{black}\bm{\nu}^2}\Big).
\end{align}
\end{corollary}
%It was expected that the model would not converge after the addition of noise due to the imperfection in communication channels. 
\subsection{Implications:}
Following Corollary \ref{coro-noisy-fedavg}, we can observe that the term corresponding to uplink noise scale as $\mathcal{O}(\frac{1}{E^{2}\sqrt{rK}})$, while the term corresponding to downlink noise is $\mathcal{O}(1)$. We can visualize the impact of both uplink and downlink noises on the MNIST dataset (refer \Cref{sec:DLE} for the details about the model architecture and other parameters) in \Cref{fig:r_E_plots}. From \cref{fig2,fig4}, we can confirm that any changes in the number of participating clients, i.e., $r$, and the local number of iterations, i.e., $E$ does not make any difference and it scales as $\mathcal{O}(1)$. Similarly, from \cref{fig1,fig3}, we inspect that the term corresponding to uplink noise scales as $\mathcal{O}(\frac{1}{E^{2}\sqrt{rK}})$. This tells us that the impact of uplink and downlink noises on convergence errors is different. Using this result, we propose to employ  SNR control strategies such that the effect of both the noises appear as \textit{higher-order terms}, not as dominant terms. Hence to curtail the impact of the noises, we want the order of the terms corresponding to the uplink and downlink noise to be $\mathcal{O}(\frac{1}{E^{1+\delta_1}K^{\frac{1}{2} + \delta_2}})$, for some $\delta_1,\delta_2 > 0$. For instance, in what follows we will adopt a strategy such that $\delta_1 = 1$ and $\delta_2 = 0.5$.

Since we have already established that the effects of both uplink and downlink noises are different, we need to employ different scaling policies for these noises. Hence, for the model to converge to an $\epsilon-$stationary point like in \texttt{FedAvg}, we need to scale down the downlink noise by $\Omega({E^2 k})$ and uplink noise by $\Omega({\sqrt{k}})$. These scaling rates result in requiring considerably less power resources compared to the prior work, e.g. \cite{wei2022federated}. Putting the requirement of strong convexity aside, the proposed policy in \cite{wei2022federated}, while consuming more power resources\footnote{Considering non-convex, $L$-smooth problems, the result in \cite{wei2022federated} seems to require $\Omega(k)$ scaling for both uplink and downlink noises.} ensures that the model converges, albeit the dominant term in the rate will depend on noise statistics. However, employing our strategy ensures that noise appears merely as a higher-order term, which means that for a large number of communication rounds, the difference between noisy and noise-free \texttt{FedAvg} will be negligible. In the discussions above we talk about the necessity of SNR scaling to ensure that the performance gracefully depends on the noise. There are works such as \cite{zheng2020design,amiri2020federated} and others that provide a practical perspective towards designing distributed systems and it is an interesting future work to implement the theoretical findings of this work in practical systems and come up with new design paradigms. A simple setting is a scenario where a set of UAVs (acting as clients in FL) need to communicate with the base station (acting as the central server in FL). Our paper's main findings indicate that the effect of the downlink channel noise is particularly detrimental, and thus, an effective strategy would involve boosting the power of the signal sent from the base station to the UAVs.

\begin{figure}[t]
% \centering
\begin{subfigure}[t]{0.5\textwidth}
    \centering
    \includegraphics[width=\textwidth]{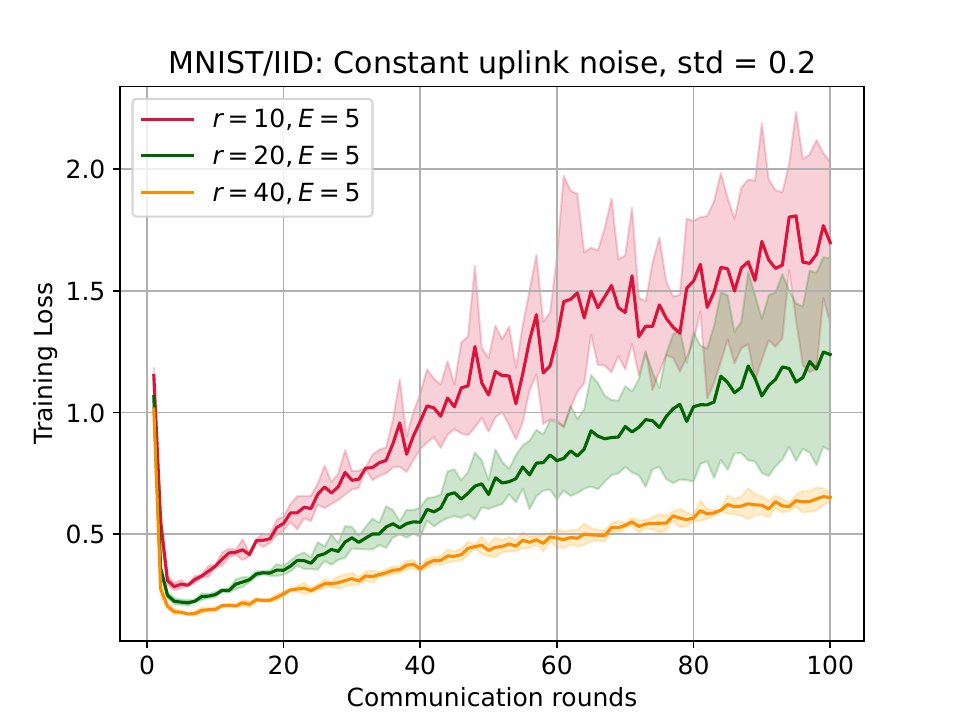}
    \caption{} \label{fig1}
\end{subfigure}
\begin{subfigure}[t]{0.5\textwidth}
    \centering
    \includegraphics[width=\textwidth]{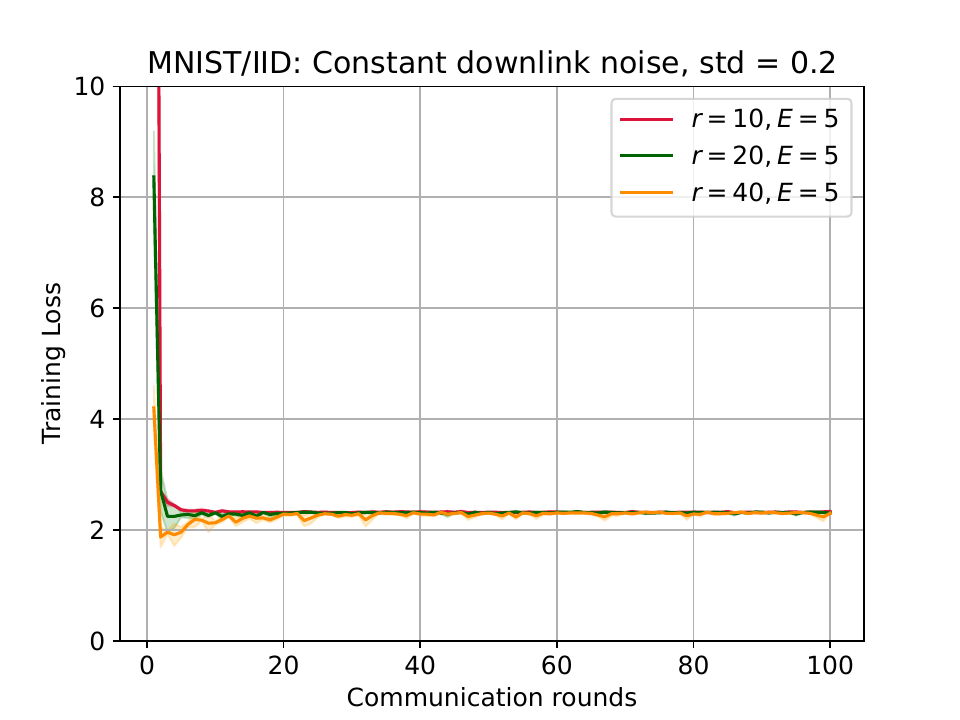}
    \caption{} \label{fig2}
\end{subfigure}\\
\begin{subfigure}[t]{0.5\textwidth}
    \centering
    \includegraphics[width=\textwidth]{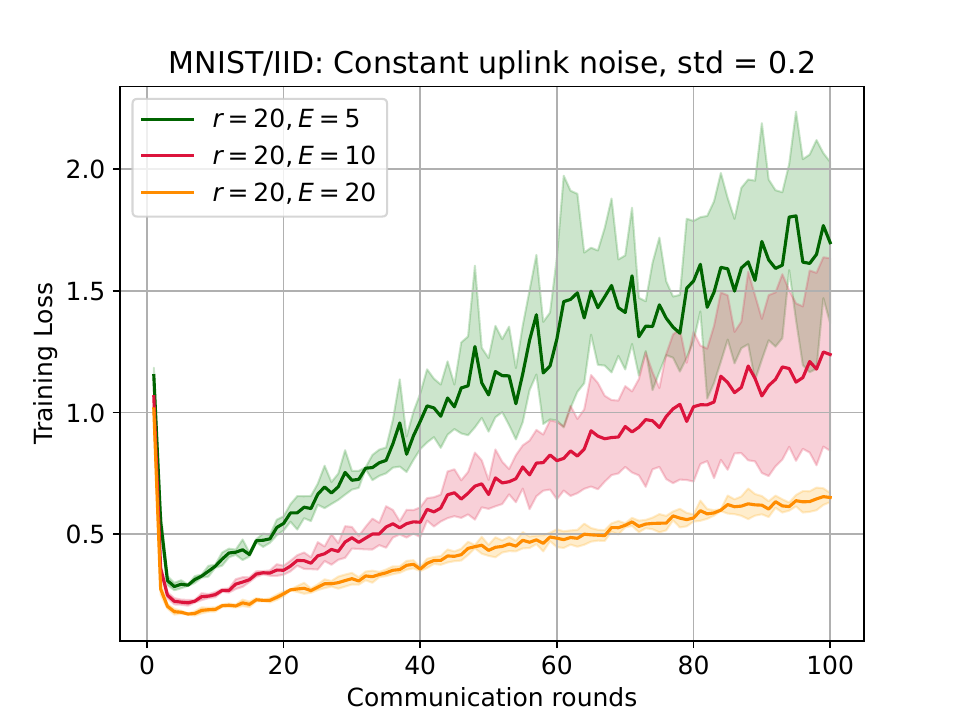}
    \caption{} \label{fig3}
\end{subfigure}
\begin{subfigure}[t]{0.5\textwidth}
    \centering
    \includegraphics[width=\textwidth]{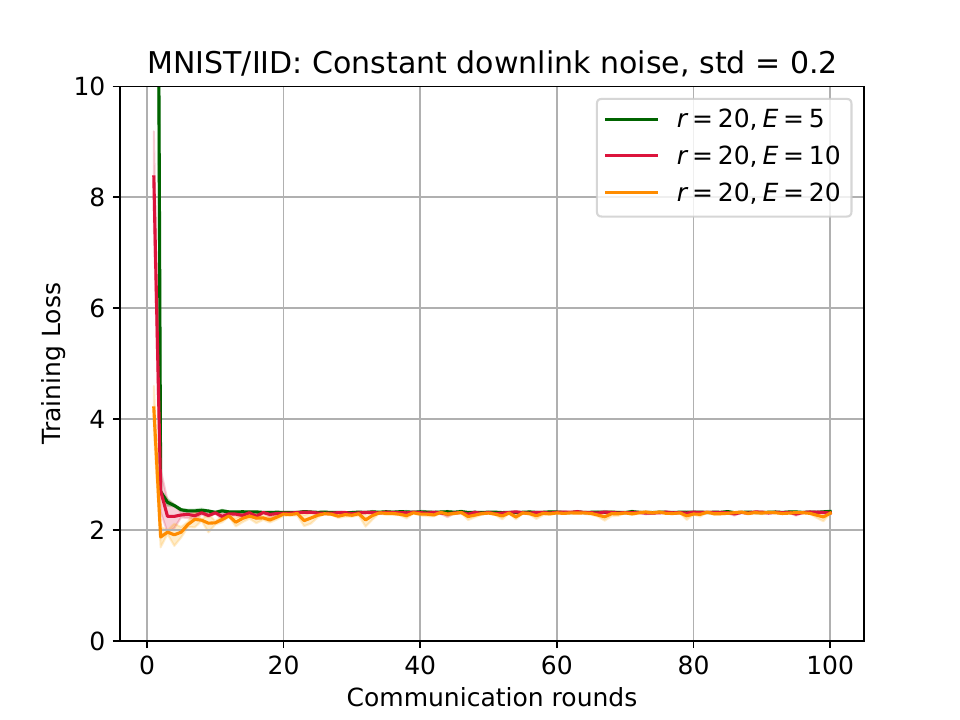}
    \caption{} \label{fig4}
\end{subfigure}
\caption{Effect of $r$ (top 2 figures) and $E$ (bottom 2 figures) on uplink and downlink noise.}
\label{fig:r_E_plots}
\vspace{-4mm}
\end{figure}
\begin{figure}[t]
\centering
    \begin{subfigure}[t]{0.5\textwidth}
    \centering
    \includegraphics[width=\textwidth]{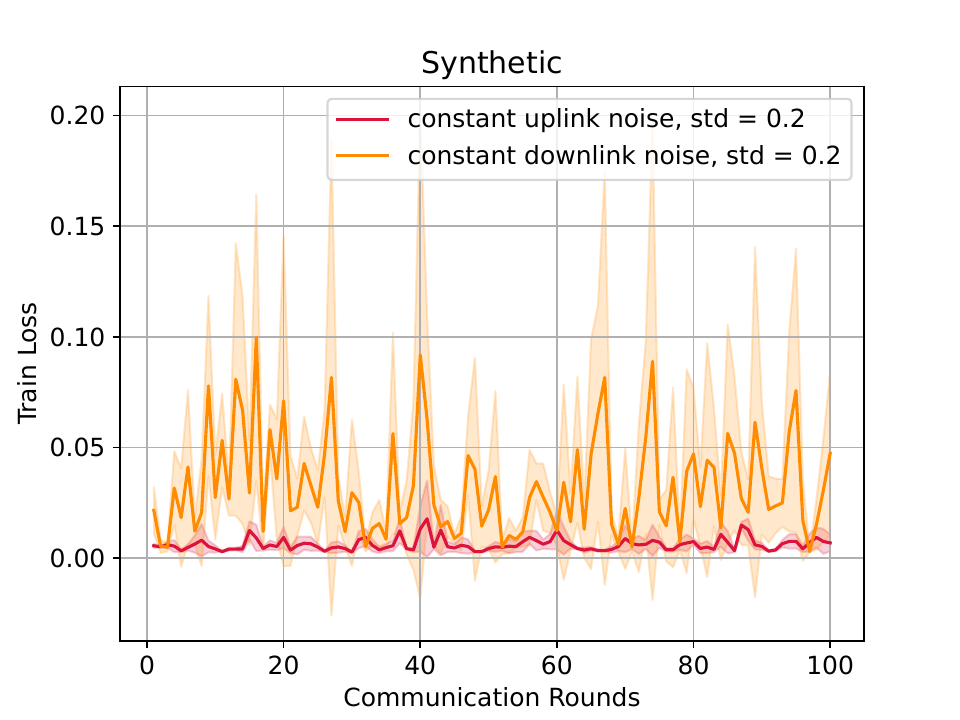}
    \caption{Constant noise} \label{fig:var_plot}
\end{subfigure}\hfill
\begin{subfigure}[t]{0.5\textwidth}
    \centering
    \includegraphics[width=\textwidth]{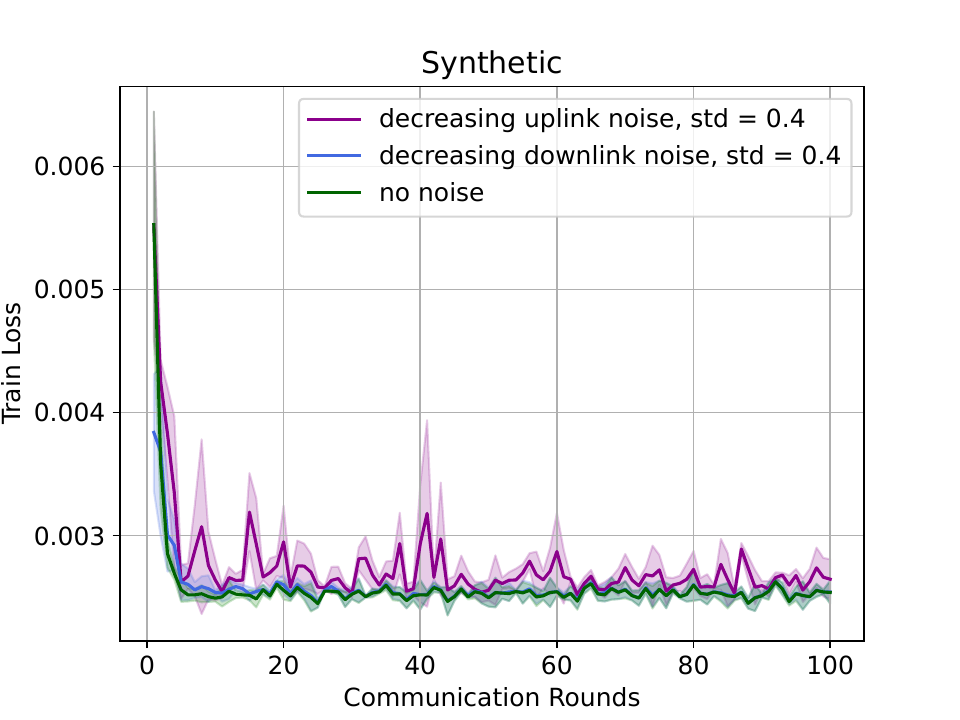}
    \caption{Decreasing noise} \label{fig:control_plot}
\end{subfigure}\hfill
\caption{Linear regression: Comparing the impact of uplink and downlink noise with (left) and without (right) SNR control.}
\label{fig:LR_plots}
\vspace{-4mm}
\end{figure}
\section{Verifying Experiments}\label{section:exp}
In this section, we demonstrate the efficacy and validity of the proposed theory through empirical analysis. For this purpose, we devise two categories of experiments, (i) synthetic experiments and (ii) deep learning experiments. 
\subsection{Synthetic experiment}
In the synthetic experiment, we train a linear regression model with m = 15000 samples. The samples $\{(\bm{x_j}, y_j)_{j=1}^m\}$ are generated based on the model $y_j = \langle\theta^{*}, \bm{x_j}\rangle + c_j$,  where $\bm{\theta}^{*} \in \mathbb{R}^{60}$, the $j^{th}$ input $\bm{x_j} \sim \mathcal{N}(0, I_{60})$, and noise $c_j \sim \mathcal{N}(0, 0.05)$. This dataset is generated such that the $(samples \times features)$ matrix has the $\ell_2$ norm of its Hessian equal to 1. These samples are then distributed over 50 clients resulting in 300 samples/client. Also, we use the mean squared error loss function.
\begin{figure}[t]
% \centering
\begin{subfigure}[t]{0.235\textwidth}
    \centering
    \includegraphics[width=\textwidth]{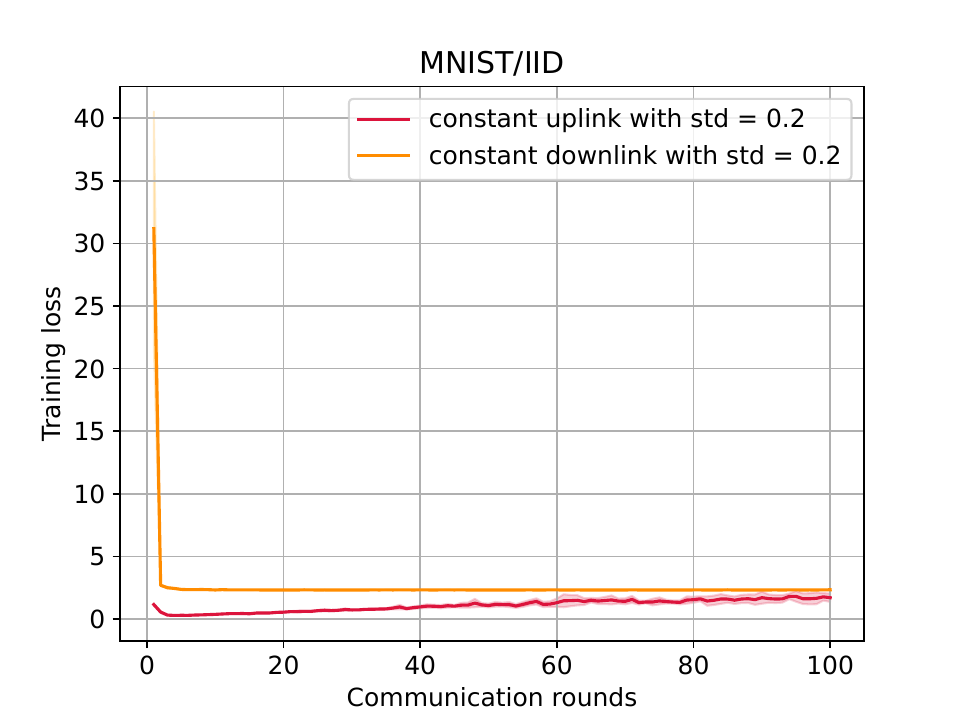}
    \caption{Constant noise addition} \label{fig1:MNIST}
\end{subfigure}
\begin{subfigure}[t]{0.235\textwidth}
    \centering
    \includegraphics[width=\textwidth]{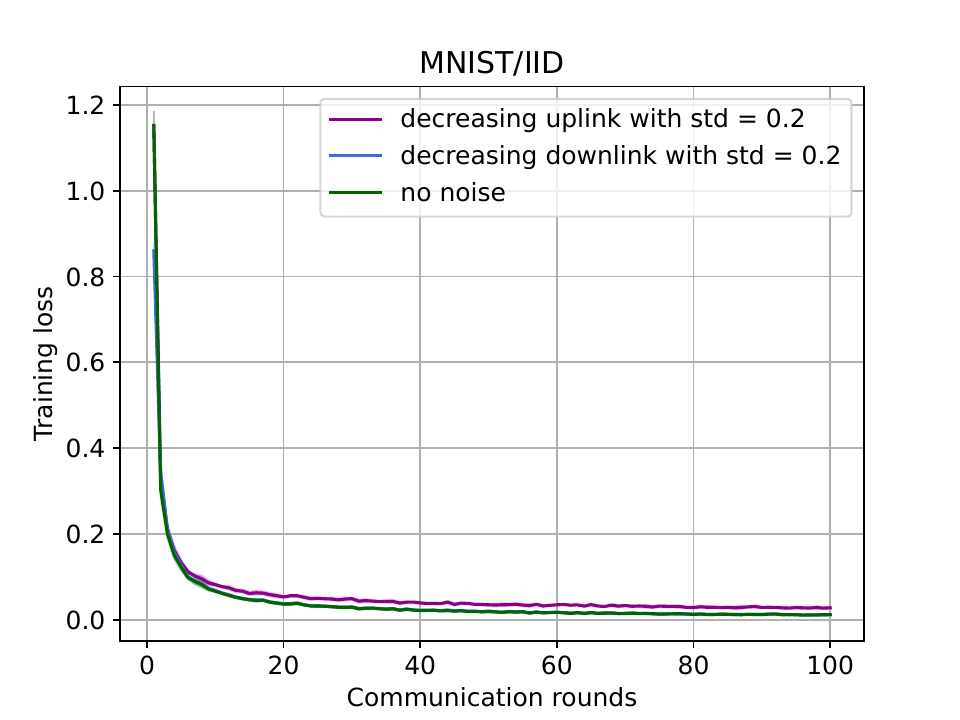}
    \caption{Controlled noise addition} \label{fig2:MNIST}
\end{subfigure}
\begin{subfigure}[t]{0.235\textwidth}
    \centering
    \includegraphics[width=\textwidth]{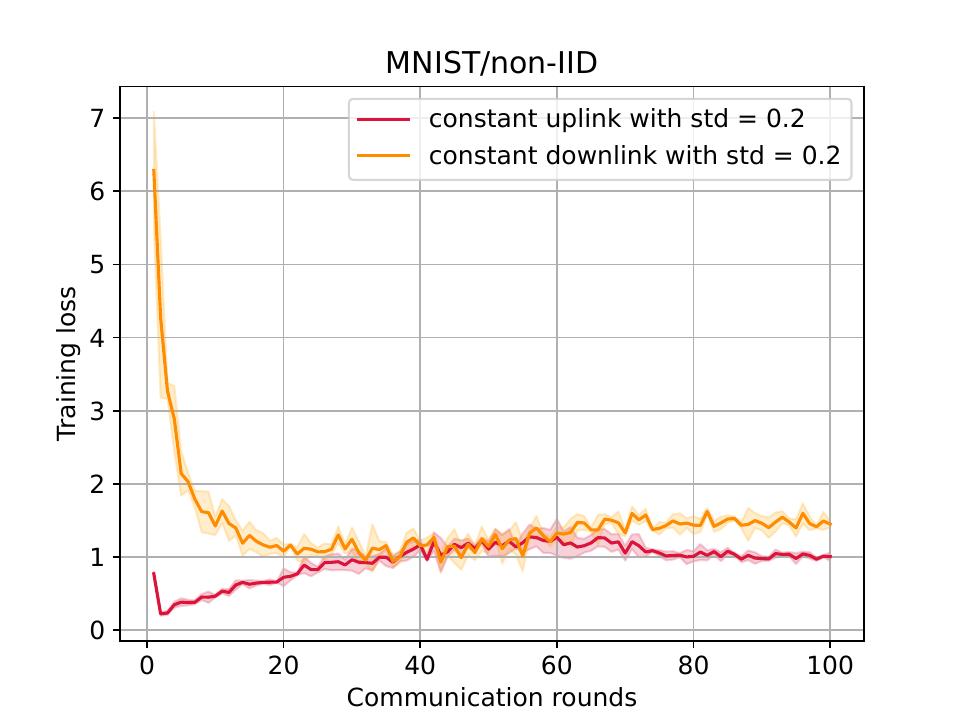}
    \caption{Constant noise addition} \label{fig3:MNIST}
\end{subfigure}
\begin{subfigure}[t]{0.235\textwidth}
    \centering
    \includegraphics[width=\textwidth]{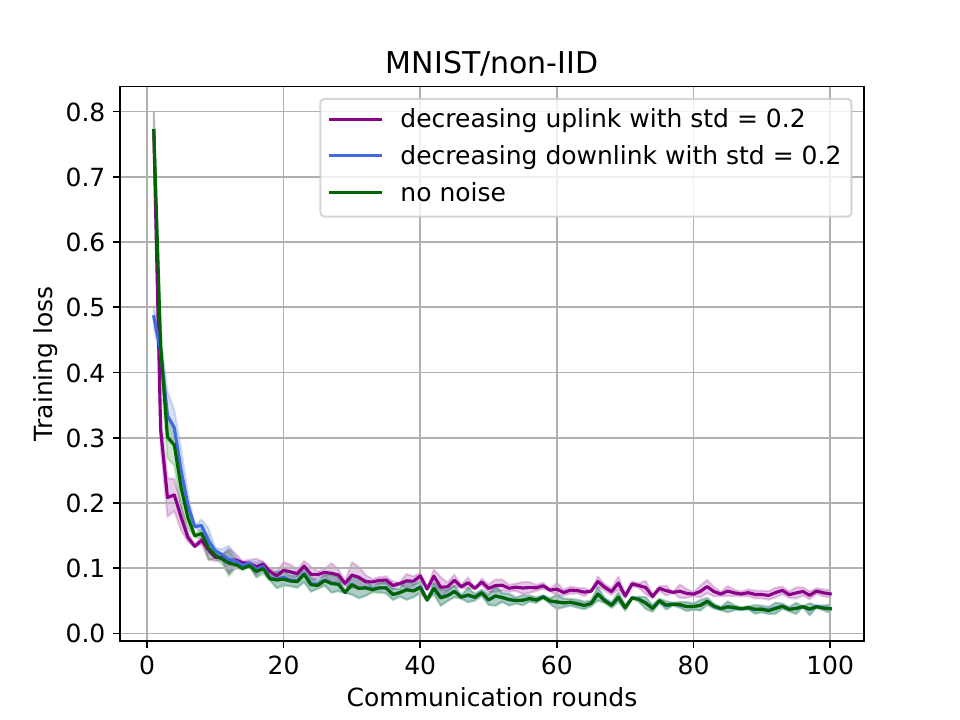}
    \caption{Controlled noise addition} \label{fig4:MNIST}
\end{subfigure}
\caption{Train loss vs. communication rounds for the IID and non-IID data distribution of the MNIST dataset. The plots (\ref{fig1:MNIST}) and (\ref{fig3:MNIST}) present the constant noise addition setting. We can see that the theory is verified from the plots. Similarly, the plots (\ref{fig2:MNIST}) and (\ref{fig4:MNIST}) show the effect of the proposed SNR control strategy. }
\label{fig:MNIST_plots}
\vspace{-4mm}
\end{figure}
\begin{figure}[t]
% \centering
\begin{subfigure}[t]{0.235\textwidth}
    \centering
    \includegraphics[width=\textwidth]{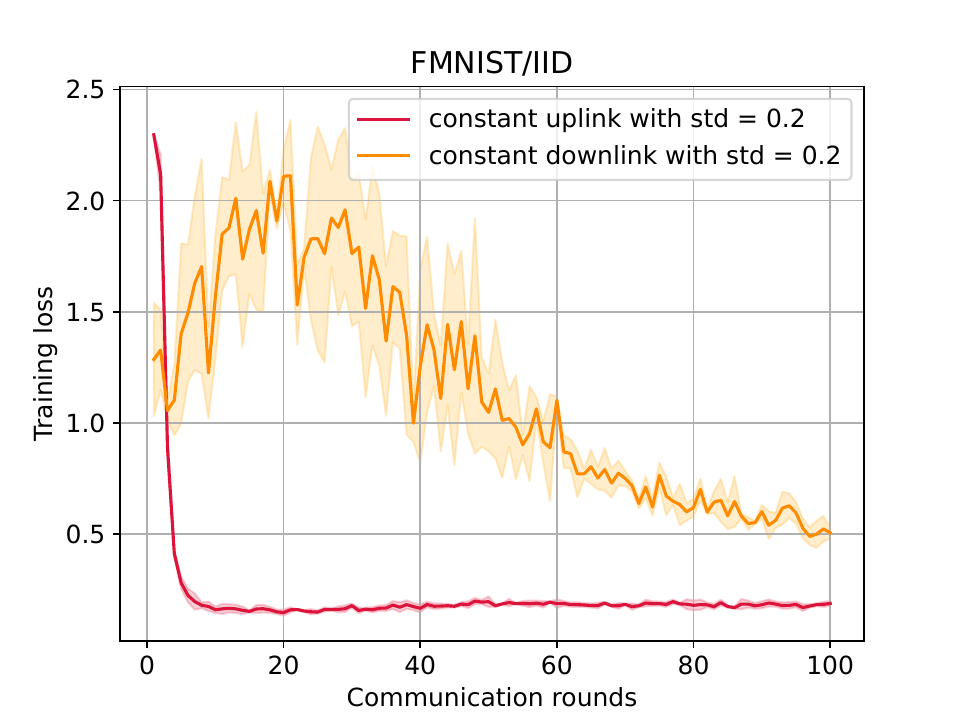}
    \caption{Constant noise addition} \label{fig1_fmnist}
\end{subfigure}
\begin{subfigure}[t]{0.235\textwidth}
    \centering
    \includegraphics[width=\textwidth]{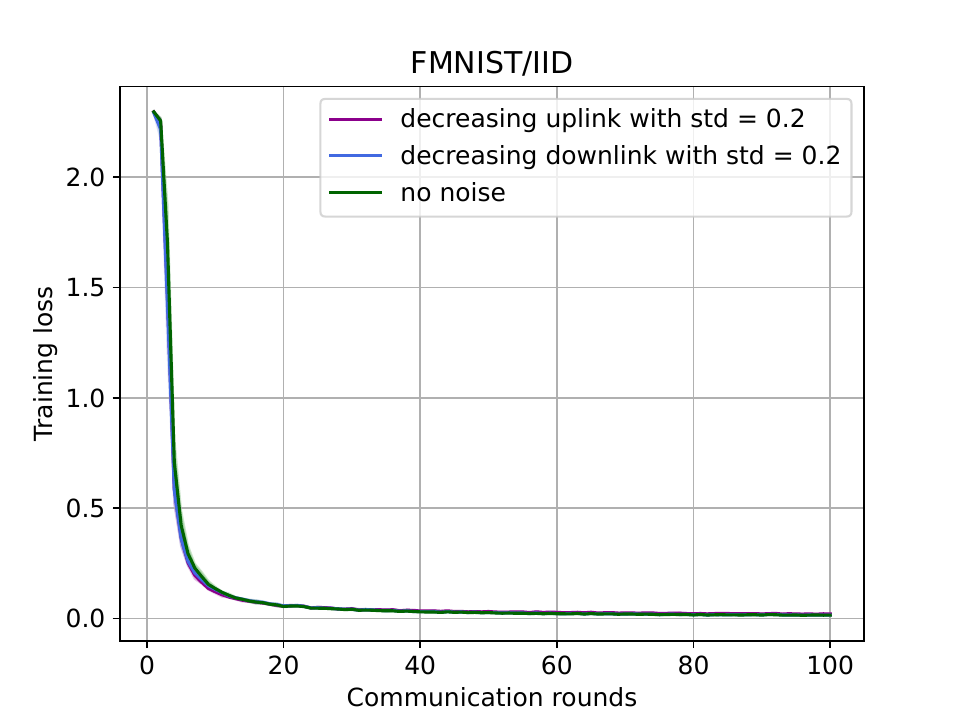}
    \caption{Controlled noise addition} \label{fig2_fmnist}
\end{subfigure}
\begin{subfigure}[t]{0.235\textwidth}
    \centering
    \includegraphics[width=\textwidth]{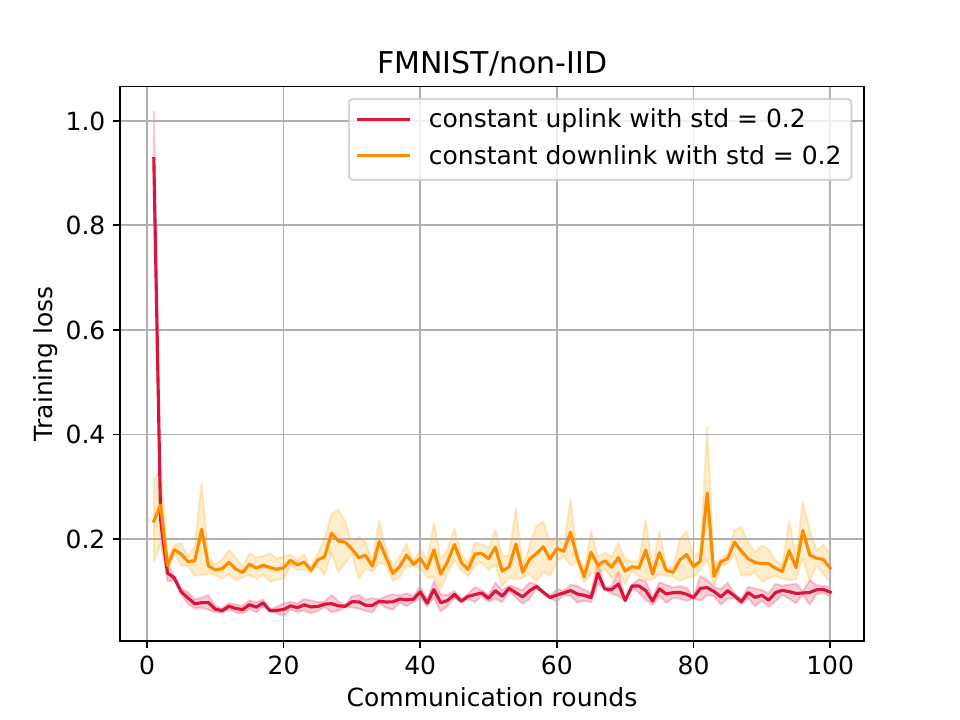}
    \caption{Constant noise addition} \label{fig3_fmnist}
\end{subfigure}
\begin{subfigure}[t]{0.235\textwidth}
    \centering
    \includegraphics[width=\textwidth]{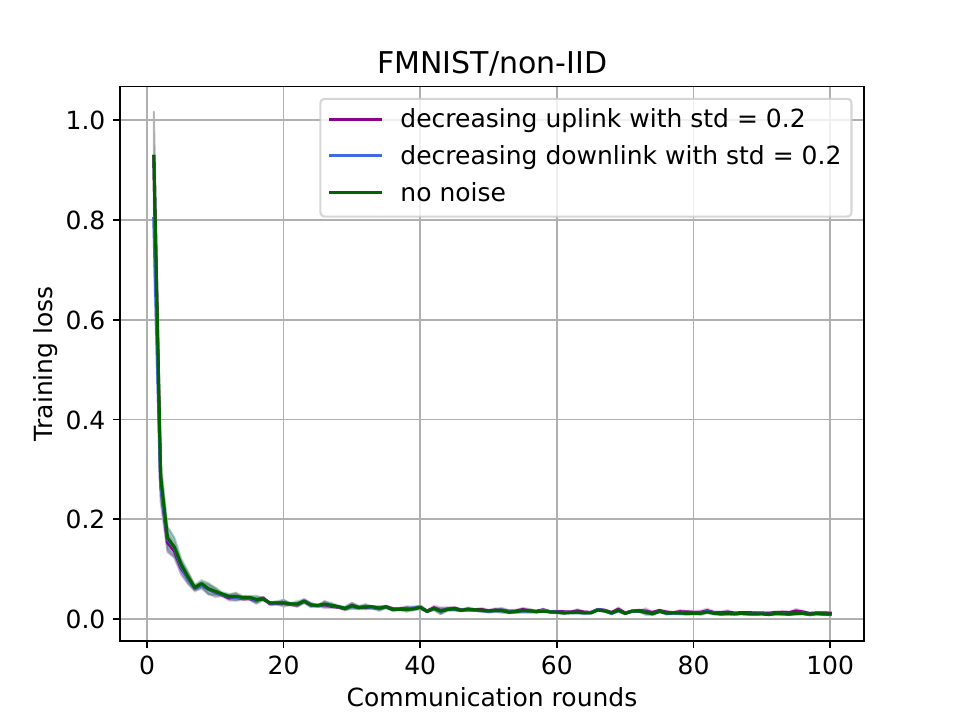}
    \caption{Controlled noise addition} \label{fig4_fmnist}
\end{subfigure}
\caption{Train loss vs. communication rounds for the IID and non-IID data distribution of the FMNIST dataset. The plots (\ref{fig1_fmnist}) and (\ref{fig3_fmnist}) present the constant noise addition setting. We can see that the theory is verified from the plots. Similarly, the plots (\ref{fig2_fmnist}) and (\ref{fig4_fmnist}) show the effect of the proposed SNR control strategy.}
\label{fig:FMNIST_plots}
\vspace{-4mm}
\end{figure}
To conduct this numerical experiment we use $n = 50$ clients and set the values of $\gamma = 18$ (see \Cref{thm-noisy-fedavg}), $L = 1$, $ E = 5$ and $K = 100$ and $BS$ (local batch size) $= 16$. In each round, $20 \%$ of the clients participate based on random selection, which leads to $r = 10$. Now from Theorem \ref{thm-noisy-fedavg}, we have $\eta_k = \frac{1}{\gamma L E}\sqrt{\frac{r}{K}}$, i.e. $\eta_k = 0.0035$.

We first consider a constant noise setting where we add both uplink and downlink noise to the communicated messages where the noises are sampled from a Gaussian distribution having zero mean i.e., $\bm{e}_k^{(i)} \sim \mathcal{N}(0, \upsilon^2)$ and $\bm{\nu}_k^{(i)} \sim \mathcal{N}(0, \upsilon^2)$, where $\upsilon = 0.2$. We visualize the impact of adding noises in \Cref{fig:var_plot}; as the figure demonstrates, consistent with the result of  \Cref{thm-noisy-fedavg}, the effect of downlink noise is more severe than the uplink noise and results in model divergence. Furthermore, the adverse effect of noise increases as $\upsilon$ increases.
\begin{figure}[t]
% \centering
\begin{subfigure}[t]{0.235\textwidth}
    \centering
    \includegraphics[width=\textwidth]{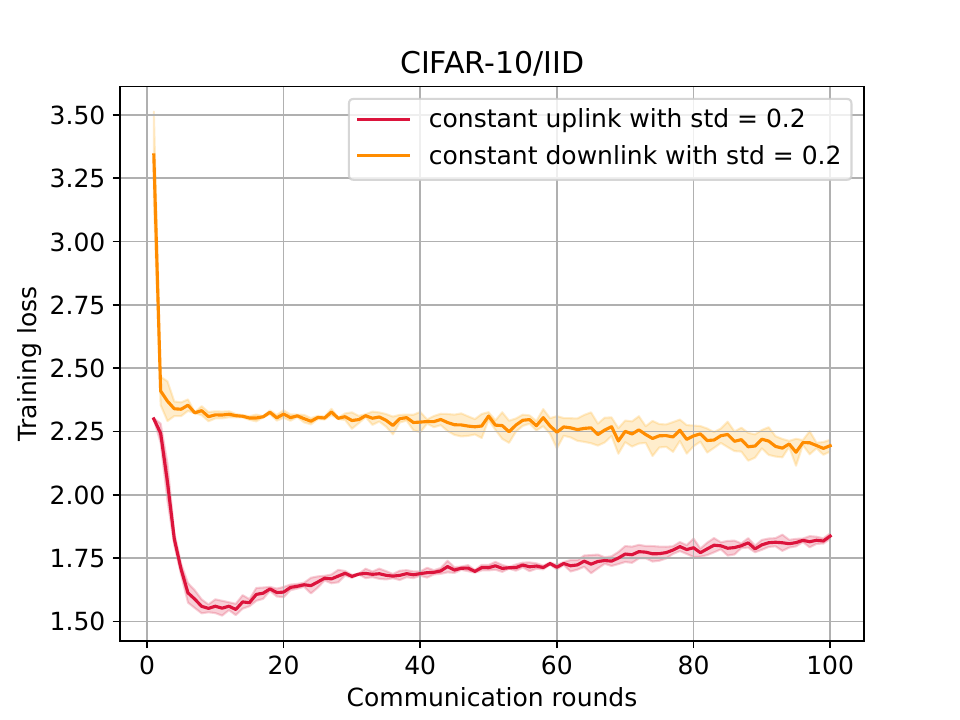}
    \caption{Constant noise addition} \label{fig1_cifar_10}
\end{subfigure}
\begin{subfigure}[t]{0.235\textwidth}
    \centering
    \includegraphics[width=\textwidth]{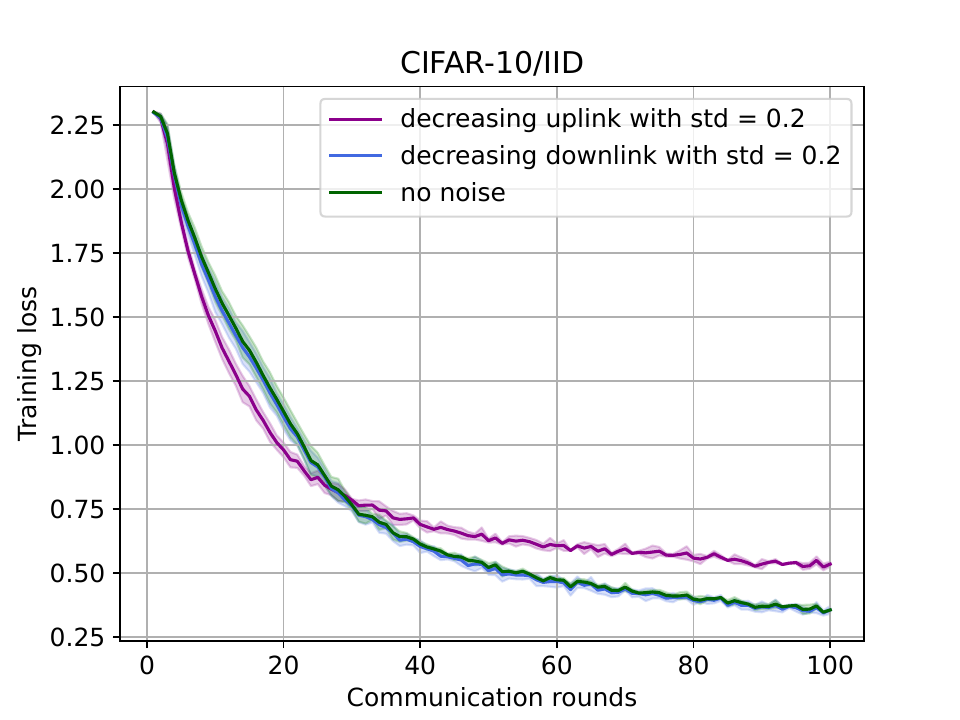}
    \caption{Controlled noise addition} \label{fig2_cifar_10}
\end{subfigure}
\begin{subfigure}[t]{0.235\textwidth}
    \centering
    \includegraphics[width=\textwidth]{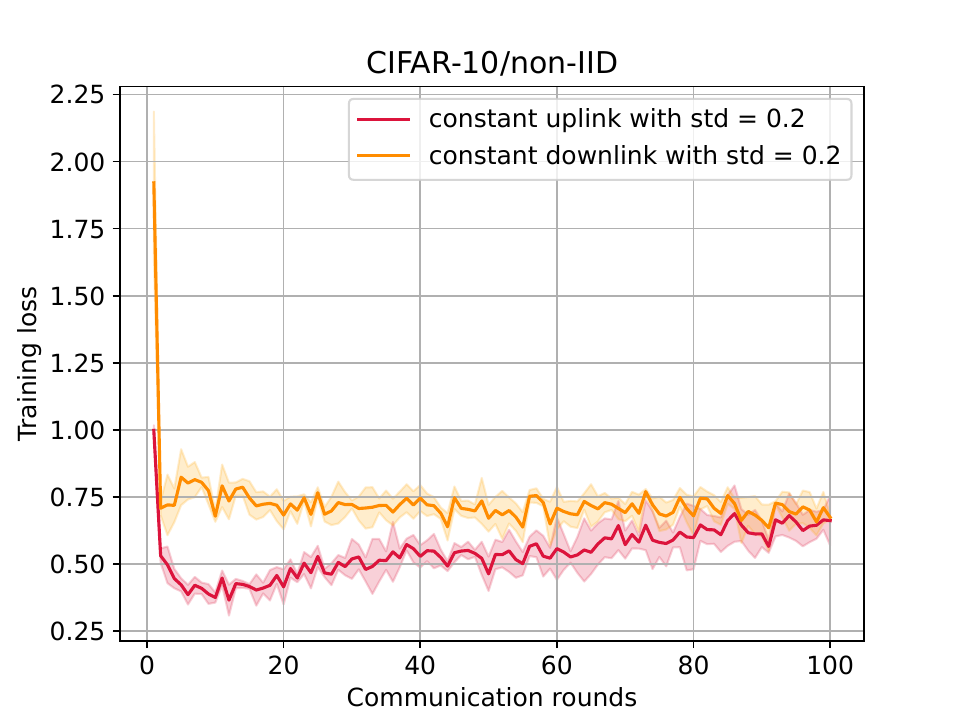}
    \caption{Constant noise addition} \label{fig3_cifar_10}
\end{subfigure}
\begin{subfigure}[t]{0.235\textwidth}
    \centering
    \includegraphics[width=\textwidth]{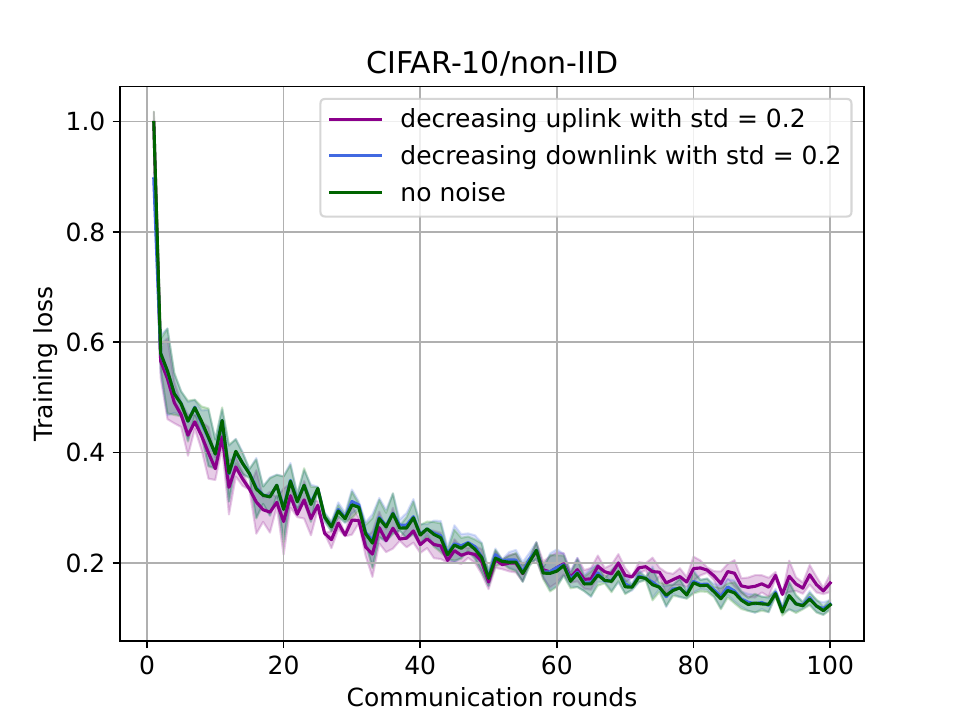}
    \caption{Controlled noise addition} \label{fig4_cifar_10}
\end{subfigure}
\caption{Train loss vs. communication rounds for the IID and non-IID data distribution of the CIFAR-10 dataset. The plots (\ref{fig1_cifar_10}) and (\ref{fig3_cifar_10}) present the constant noise addition setting. We can see that the theory is verified from the plots. Similarly, the plots (\ref{fig2_cifar_10}) and (\ref{fig4_cifar_10}) show the effect of the proposed SNR control strategy.}
\label{fig:CIFAR_10_plots}
\vspace{-4mm}
\end{figure}
Now, we test the efficacy of the proposed SNR control strategy in \Cref{sec:snr}. In particular, since we already established using Theorem \ref{thm-noisy-fedavg} that the effect of downlink noise is more degrading than uplink noise, we can utilize different SNR scaling policies to save on power resources while alleviating the effect of noise. Hence, we  scale the downlink and uplink noises  by $\Omega(\frac{1}{E^2 k})$ and  $\Omega(\frac{1}{\sqrt{k}})$, respectively. We can observe the results from Fig. \ref{fig:control_plot} and see how it converges almost in tandem with the noise-free case, as predicted by \Cref{coro-noisy-fedavg}.
\begin{figure*}[ht]
\begin{subfigure}[t]{0.33\textwidth}
    \centering
    \includegraphics[width=\textwidth]{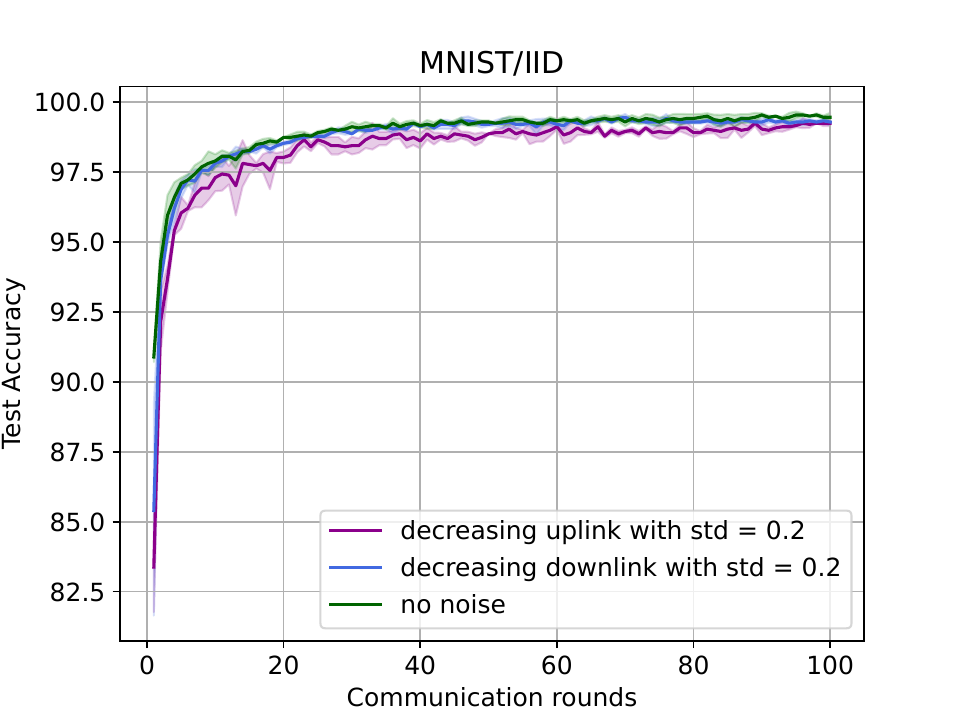}
    \label{fig1_test_mnist_iid}
\end{subfigure}
\begin{subfigure}[t]{0.33\textwidth}
    \centering
    \includegraphics[width=\textwidth]{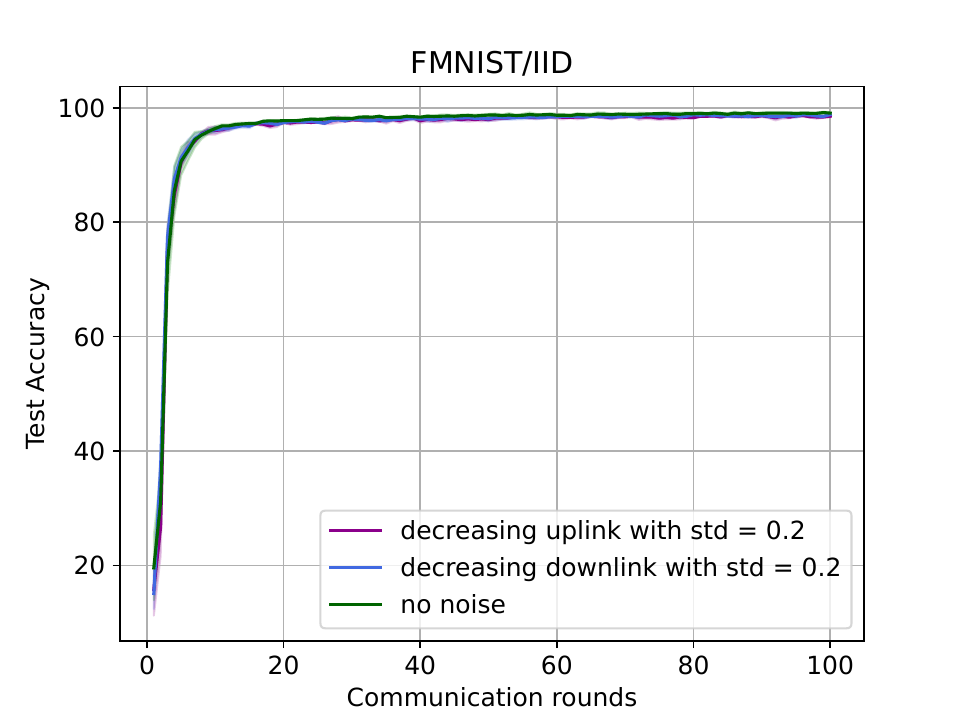}
    \label{fig2_test_fmnist_iid}
\end{subfigure}
\begin{subfigure}[t]{0.33\textwidth}
    \centering
    \includegraphics[width=\textwidth]{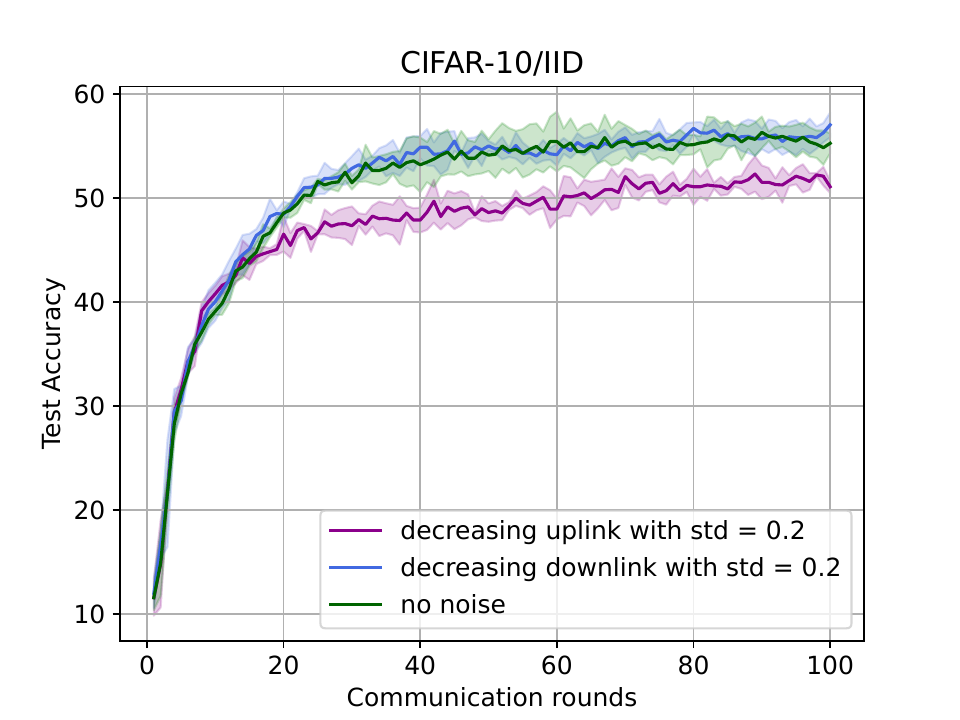}
    \label{fig3_test_cifar_10_iid}
\end{subfigure}
\\
\begin{subfigure}[t]{0.33\textwidth}
    \centering
    \includegraphics[width=\textwidth]{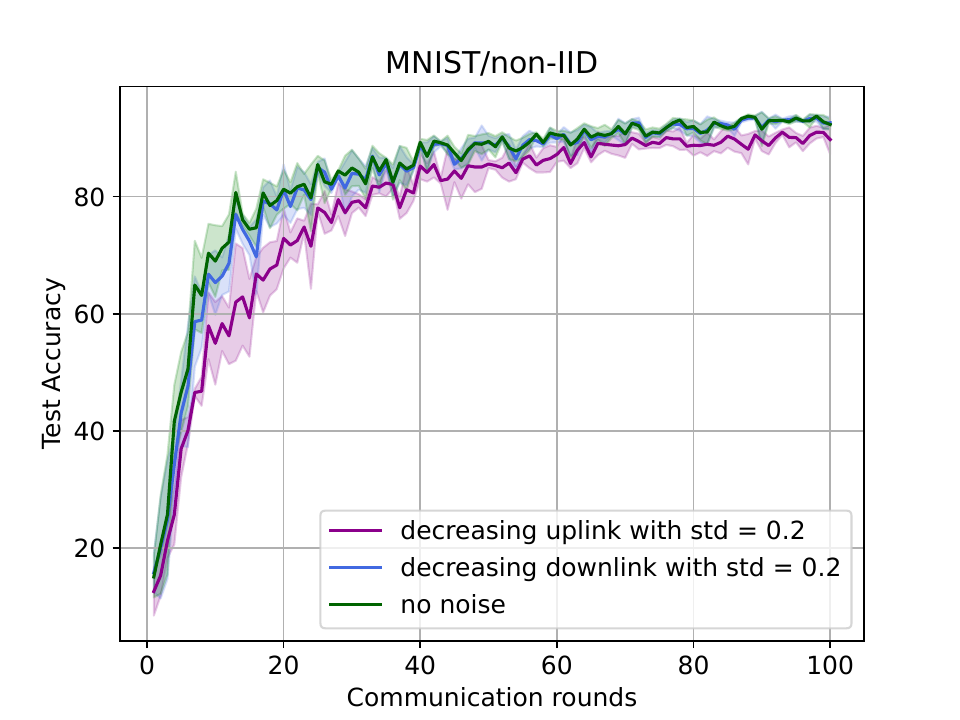}
    \label{fig1_test_mnist_non_iid}
\end{subfigure}
\begin{subfigure}[t]{0.33\textwidth}
    \centering
    \includegraphics[width=\textwidth]{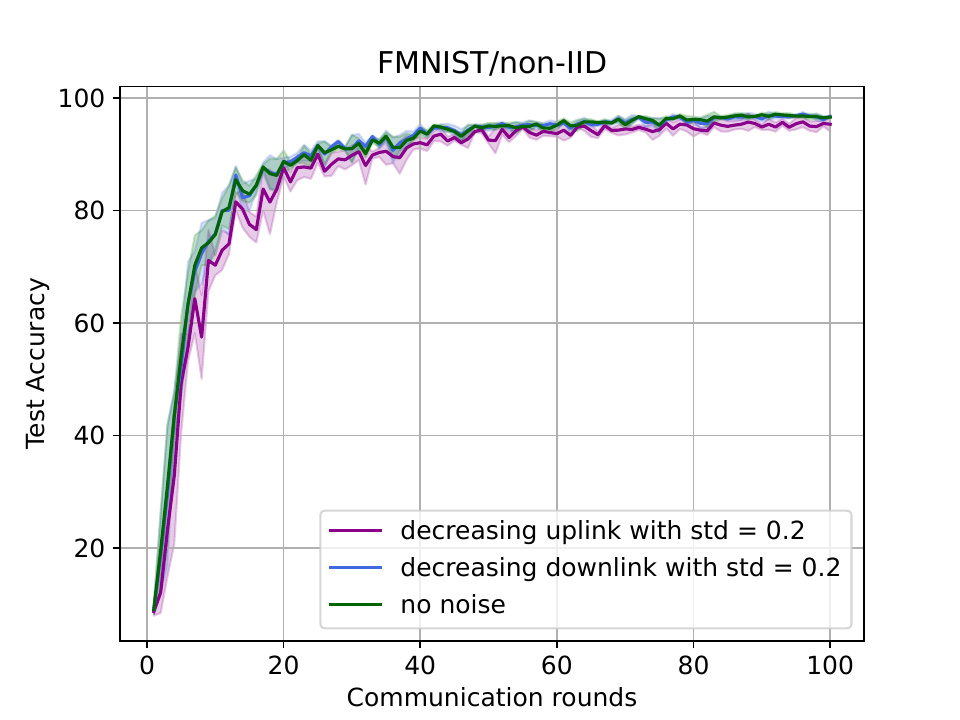}
    \label{fig2_test_fmnist_non_iid}
\end{subfigure}
\begin{subfigure}[t]{0.33\textwidth}
    \centering
    \includegraphics[width=\textwidth]{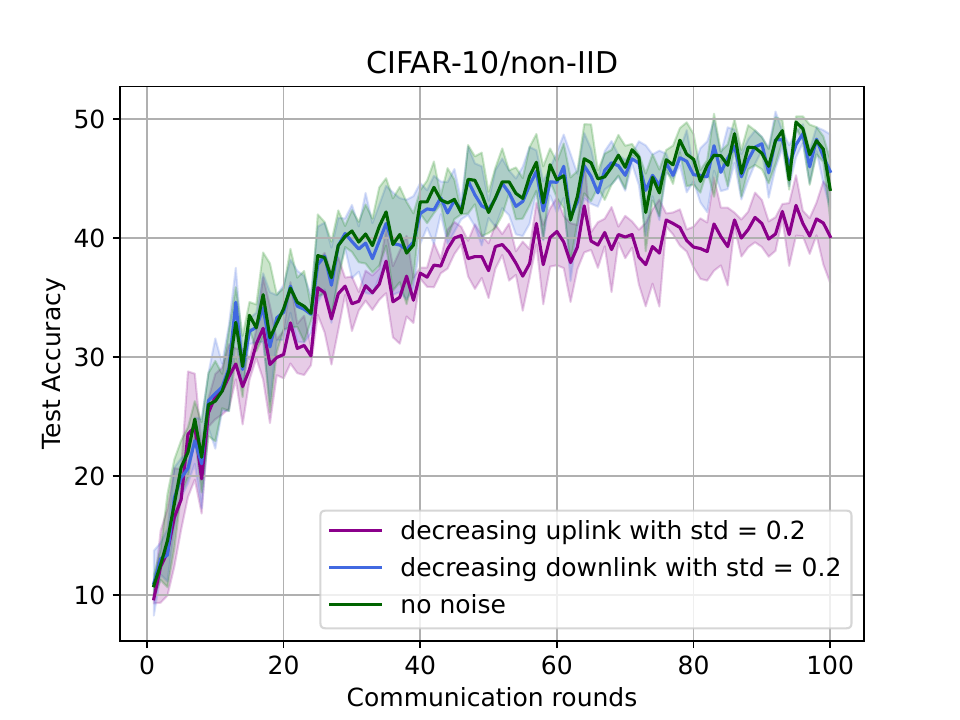}
    \label{fig3_test_cifar_10_non_iid}
\end{subfigure}
\caption{Test accuracy vs. communication rounds for the IID (top 3 figures) and non-IID (bottom 3 figures) data distribution of the MNIST, FMNIST, and CIFAR-10 datasets.}
\label{fig:test_acc_plots}
\vspace{-4mm}
\end{figure*}
\subsection{Deep learning experiment}
\label{sec:DLE}
To check the validity of our theory on real-world datasets, we run deep learning experiments on the MNIST, CIFAR-10, and Fashion-MNIST (FMNIST) datasets, both for IID and non-IID settings. Before we start with the different experimental setups we clarify that these experiments are not designed to provide benchmark results on the corresponding dataset but rather provide insights into the proposed theory.
\subsubsection{MNIST}
In this, we train a CNN model with 60000 samples, equally distributed over a set of $n = 100$ clients. In each round, $20\%$ of the clients participate based on random selection, which leads to $r = 20$. To emulate the IID setting, the data is shuffled and then randomly assigned to each client resulting in 600 samples/client. For the non-IID setting, we assign 1 or 2 labels to each client randomly. The model in each client is two $5 \times 5$ convolution layers, having 32 and 64 channels, respectively. Each of these layers is followed by a $2\times2$ max pooling. Finally, the output is fed to a fully connected layer with 512 units followed by a ReLU activation, and a final output layer with softmax. We also included a dropout layer having the dropout = 0.2. The following are the parameters used for the training: local number of iterations, $ E = 5$, global communication rounds, $K = 100$, local batch size, $BS= 20$, and learning rate, $\eta_k = 0.01$ for IID case and $\eta_k = 0.001$ for the non-IID case.

\subsubsection{Fashion-MNIST(FMNIST)}
Here we train a CNN model with 60000 samples. The data is distributed over a set of $n=100$ clients, where, $r=20$ clients participate in each round randomly. The experimental premise is divided into the IID and non-IID settings based on the distribution of the dataset. In IID the dataset is equally distributed equally amongst all the clients and in the non-IID setting, we only share 1 or 2 labels with each client to maintain high data heterogeneity. In each client, the model consists of two convolutional layers and three fully-connected layers. These $5 \times 5$ convolution layers have 6 and 12 channels respectively. Each of these layers is followed by ReLU activation and a $2 \times 2$ max pooling layer. The following are the parameters used for the training: local number of iterations, $ E = 5$, global communication rounds, $K = 100$, local batch size, $BS= 20$, and learning rate, $\eta_k = 0.01$.
\subsubsection{CIFAR-10}
We train a CNN model similar to MNIST and FMNIST. The 50000 samples are equally distributed over $n=100$ clients out of which $20\%$, i.e., $r=20$, clients participate randomly in each round. In the case of the IID setting the data is shuffled and then randomly assigned to each client whereas in the case of the non-IID setting only 1 or 2 labels are assigned to each client randomly. The model architecture consists of two convolutional layers, followed by three fully-connected layers. The first convolutional layer has 3 input channels and 6 output channels, with a kernel size of 5. The second convolutional layer has 6 input channels and 16 output channels, also with a kernel size of 5. Between the convolutional layers, there is a max pooling layer with a kernel size of 2 and a stride of 2. After the convolutional layers, the output is flattened and passed through the three fully-connected layers, with 120, 84, and 10 units respectively. The forward function applies a ReLU activation function after each convolutional and fully-connected layer, except for the final fully-connected layer. The following are the parameters used for the training: local number of iterations, $ E = 5$, global communication rounds, $K = 100$, local batch size, $BS= 20$, and learning rate, $\eta_k = 0.01$.
\subsubsection{FEMNIST}
In this, we train a CNN model with $n = 1000$ clients on the FEMNIST \cite{caldas2018leaf} dataset where in each round only $1\%$, i.e., 10 clients participate. This dataset consists of images that are highly heterogeneous since it has 62 classes consisting of digits and upper and lower case English characters. Additionally, the heterogeneity comes from the fact that these characters are written by different subjects. Thus, we only conduct experiments for the non-IID case. The model architecture consists of two convolutional layers followed by max-pooling, dropout, and fully-connected layers. The model architecture is adapted from \cite{reddi2020adaptive}. The following are the parameters used for the training: local number of iterations, $ E = 5$, global communication rounds, $K = 100$, local batch size, $BS= 50$, and learning rate, $\eta_k = 0.01$.
\begin{figure}[t]
\centering
    \begin{subfigure}[t]{0.33\textwidth}
    \centering
    \includegraphics[width=\textwidth]{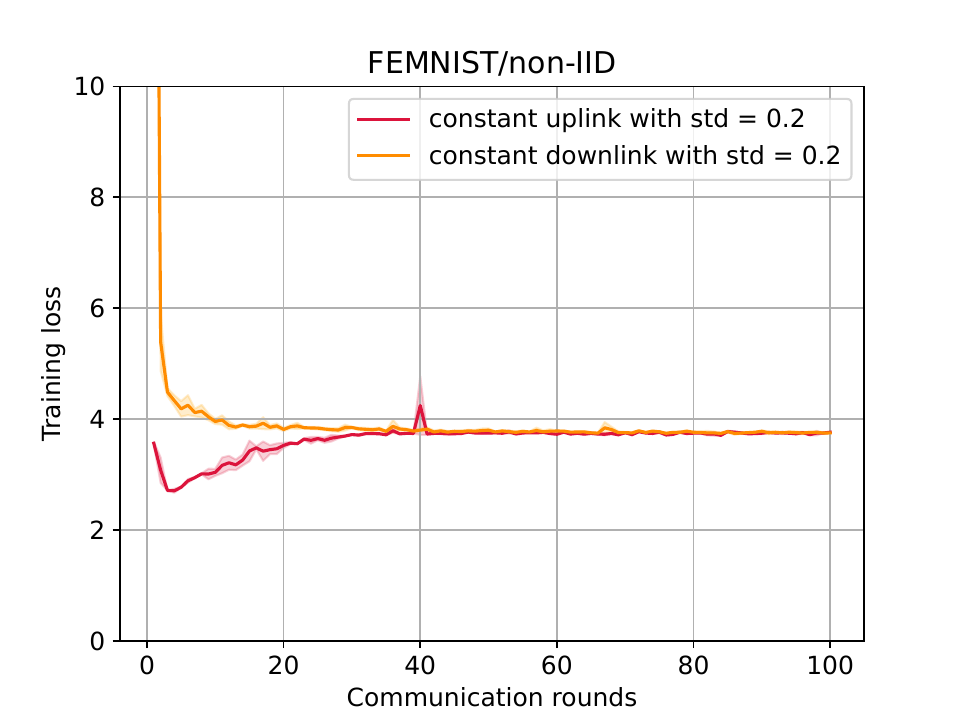}
    \caption{Constant noise} \label{fig:femnist_var_plot}
\end{subfigure}\hfill
\begin{subfigure}[t]{0.33\textwidth}
    \centering
    \includegraphics[width=\textwidth]{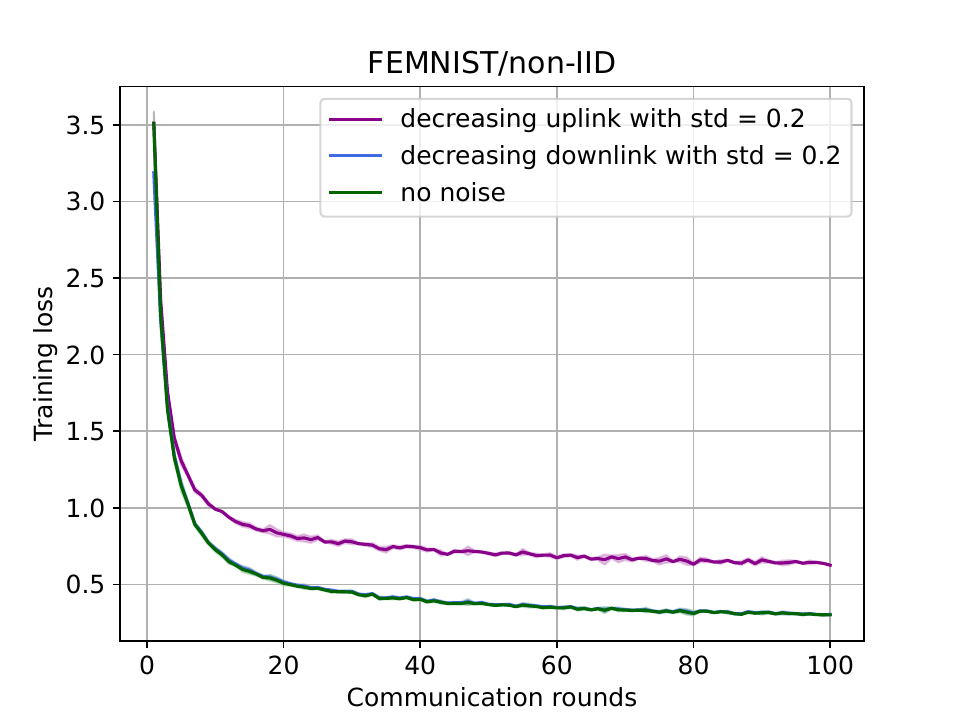}
    \caption{Decreasing noise} \label{fig:femnist_control_plot}
\end{subfigure}\hfill
\begin{subfigure}[t]{0.33\textwidth}
    \centering
    \includegraphics[width=\textwidth]{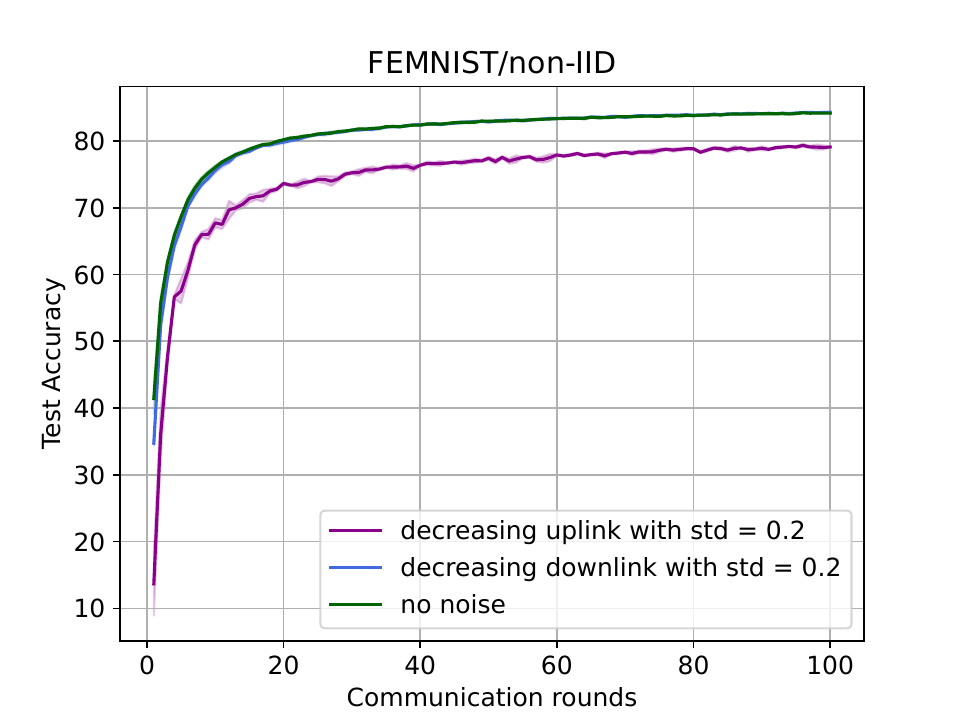}
    \caption{Test accuracy} 
    \label{fig:femnist_test}
\end{subfigure}\hfill
\caption{Train loss and test accuracy vs. communication rounds for the non-IID data distribution of the FEMNIST dataset. The plots (\ref{fig:femnist_var_plot}) present the constant noise addition setting. We can see that the theory is verified from the plots. Similarly, the plots (\ref{fig:femnist_control_plot} and \ref{fig:femnist_test}) show the effect of the proposed SNR control strategy.}
\label{fig:femnist_plots}
\vspace{-4mm}
\end{figure}
% \begin{figure}[t]
% \includegraphics[width=1\linewidth]{NF_IEEE_ACCESS/Plots/FEMNIST/Test_FEMNIST_non__IID_dec_acc.pdf}
% \caption{Test accuracy vs. communication rounds for the non-IID data distribution of FEMNIST dataset.}
% \label{fig:femnist_test}
% \vspace{-4mm}
% \end{figure}

For all four (MNIST, FMNIST, CIFAR-10, and FEMNIST) experiments, we follow the same experimental setting as the synthetic experiment, i.e., the uplink and downlink noises are sampled from a zero mean Gaussian distribution $\sim \mathcal{N}(0, \upsilon^2)$. For our experiments, we choose $\upsilon = 0.2$. Again, to imitate the noisy transmission channels, we add both uplink and downlink noise to the communicated messages. The results are shown in \Cref{fig:MNIST_plots,fig:FMNIST_plots,fig:CIFAR_10_plots,fig:test_acc_plots,fig:femnist_plots} and are generated by averaging over 3 independent runs. We can visualize from the figure that the effect of noises inhibits the model from converging. One noteworthy point in the case of the FEMNIST dataset is the overlap of uplink and downlink noises after certain rounds. We hypothesize that this discrepancy goes back to the features of this experiment which is highly heterogeneous in nature. Also, the model and thus the objective function is not smooth and that is why we see this gap. So, essentially the effect of having a non-smooth function is exacerbated in this particular experiment due to the high level of heterogeneity. However, by employing our proposed SNR control policy in \Cref{sec:snr}, we achieve model convergence for both noises with negligible convergence error with respect to the noise-free case.
\section{Conclusion}
We studied the effects of having imperfect/noisy communication channels for federated learning. To the best of our knowledge, this paper is the first to establish the convergence analysis of FL where consideration has been made on both noisy transmission channels and smooth non-convex loss function without requiring the restrictive and hard-to-verify assumption of bounded client dissimilarity. By analyzing the convergence with these relaxed assumptions, we theoretically demonstrated that the effect of downlink noise is more detrimental than uplink noise. Using this insight, we proposed to employ SNR scaling policies for respective noisy channels that result in considerable savings in power consumption compared to existing approaches. We verified these theoretical findings via empirical results demonstrating the efficacy of the proposed analysis and its validity. 
%Due to its applicability in a real-world scenario, this work should be improved further. 
Future work may involve investigating a parameter-free version of this scenario, i.e., an FL scheme that does not require the knowledge of parameters such as smoothness and analyzing its implication on the design of the system.

\appendix[Lemmas and Proofs]
\section{\break Lemmas and Proofs}
\label{appendix:lemma-NFL}
\begin{lemma}
\label{sep26-lem3}
For $\eta_k L E \leq \frac{1}{2}$, we have:
\begin{multline*}
    \mathbb{E}[f(\bm{w}_{k+1})] 
    \leq \mathbb{E}[f(\bm{w}_k)] - \frac{\eta_k (E-1)}{2} \mathbb{E}[\|\nabla f(\bm{w}_k)\|^2]
    + 4\eta_k^2 L E^2 \Big(\frac{(n-r)}{r(n-1)} + \frac{2}{3}\eta_k L E\Big)\Big(\frac{1}{n}\sum_{i \in [n]} \mathbb{E}[\|\nabla {f}_i(\bm{w}_{k})\|^2]\Big) 
    \\
    + \eta_k^2 L E \Big(\frac{\eta_k L E}{n}\Big(1 + \frac{2nE}{3} + n\Big) + \frac{1}{r} + \frac{(n-r)}{r(n-1)}\Big)\sigma^2
    + \frac{\eta^2_kL}{2r}\frac{1}{n}\sum_{i \in [n]}\bm{U}^2_{k,i}
    + \frac{\eta_kL^2}{2}\Big(1 + 2\eta_kL + 4E\{1+3\eta^2_kL^2
    \\
    + 2\eta_kLE(2+3\eta^2_kL^2)(\frac{2}{3}\eta_kLE +\frac{(n-r)}{r(n-1)})\}\Big)\frac{1}{n}\sum_{i \in [n]}\bm{N}^2_{k,i}.
\end{multline*}
\end{lemma}

\begin{proof}
Define
\begin{align*}
    \widehat{\bm{u}}_{k,\tau}^{(i)} := \widetilde\nabla {f}_i(\bm{w}^{(i)}_{k, \tau}; \mathcal{B}^{(i)}_{k, \tau})\text{,  } & \widehat{\bm{u}}_{k,\tau} := \frac{1}{n}\sum_{i \in [n]} \widehat{\bm{u}}_{k,\tau}^{(i)}\text{, }
    \\
    \bm{u}_{k,\tau} := \frac{1}{n}\sum_{i \in [n]}\nabla f_i(\bm{w}^{(i)}_{k, \tau})\text{,  } & 
    \overline{\bm{w}}_{k,\tau} := \frac{1}{n}\sum_{i \in [n]}\bm{w}^{(i)}_{k, \tau}.
\end{align*}
Then:
\begin{equation}
    \label{eq:feb28-101}
    \nabla {f}(\bm{w}_k ) = \frac{1}{n}\sum_{i \in [n]}\nabla{f}_i(\bm{w}_k )
\end{equation}
\begin{align}
    \label{eq:feb28-1}
    \bm{w}_{k+1} = \bm{w}_k - \eta_k\Big[\frac{1}{r}\sum_{i \in \mathcal{S}_k}\Big( \bm{e}^{(i)}_k +  \sum_{\tau=0}^{E-1}\widehat{\bm{u}}_{k,\tau}^{(i)} + \widetilde\nabla {f}_i(\bm{w}_k + \bm{\nu}_k^{(i)}; \mathcal{B}^{(i)}_{k, 0}) 
    - \widetilde\nabla {f}_i(\bm{w}_k ; \mathcal{B}^{(i)}_{k, 0})\Big)\Big].
\end{align}
\begin{align}
    \label{eq:feb28-102}
    {\bm{w}}_{k,\tau}^{(i)} = \bm{w}_k + \bm{\nu}_k^{(i)} - \eta_k\Big( \sum_{t=0}^{\tau-1}\widehat{\bm{u}}_{k,\tau}^{(i)} + \widetilde\nabla {f}_i(\bm{w}_k + \bm{\nu}_k^{(i)}; \mathcal{B}^{(i)}_{k, 0}) - \widetilde\nabla {f}_i(\bm{w}_k ; \mathcal{B}^{(i)}_{k, 0})\Big).
\end{align}
\begin{align}
    \label{eq:feb28-1-0}
    \overline{\bm{w}}_{k,\tau} = \bm{w}_k + \frac{1}{n}\sum_{i \in [n]}\bm{\nu}_k^{(i)} - \eta_k\Big[\sum_{t=0}^{\tau-1}\widehat{\bm{u}}_{k,t} +\frac{1}{n}\sum_{i \in [n]}\Big( \widetilde\nabla {f}_i(\bm{w}_k + \bm{\nu}_k^{(i)}; \mathcal{B}^{(i)}_{k, 0}) - \widetilde\nabla {f}_i(\bm{w}_k ; \mathcal{B}^{(i)}_{k, 0})\Big)\Big].
\end{align}
\begin{equation}
    \label{eq:feb28-2}
    \mathbb{E}_{\{\mathcal{B}^{(i)}_{k, \tau}\}_{i=1}^n}[\widehat{\bm{u}}_{k,\tau}] = \bm{u}_{k,\tau}.
\end{equation}
\begin{equation}
    \label{eq:feb28-3}
    \mathbb{E}\Big[\Big\|\sum_{t=0}^{\tau-1}\widehat{\bm{u}}_{k,t}\Big\|^2\Big] \leq \tau \sum_{t=0}^{\tau-1}\mathbb{E}[\|\bm{u}_{k,t}\|^2] + \frac{\tau \sigma^2}{n}.
\end{equation}
\begin{equation}
    \label{eq:feb28-3-1}
    \mathbb{E}\Big[\Big\|\sum_{t=0}^{\tau-1}\widehat{\bm{u}}_{k,t}^{(i)}\Big\|^2\Big] \leq \tau \sum_{t=0}^{\tau-1}\mathbb{E}[\|\nabla f_i(\bm{w}^{(i)}_{k, t})\|^2] + {\tau \sigma^2}.
\end{equation}
\begin{equation}
    \label{eq:feb28-3-2}
    \bm{N}^{2}_{k, i} := \mathbb{E}\Big[\Big\|\bm{\nu}^{i}_k\Big\|^2\Big] .
\end{equation}
\begin{equation}
    \label{eq:feb28-3-3}
    \bm{U}^2_{k, i} := \mathbb{E}\Big[\Big\|\bm{e}^{i}_k\Big\|^2\Big].
\end{equation}
Recall that $\sigma^2$ is the maximum variance of the local (client-level) stochastic gradients.
In \cref{eq:feb28-3}, the expectation is w.r.t. $\{\mathcal{B}^{(i)}_{k, t}\}_{i=1, t=0}^{n, \tau-1}$ and it follows due to the independence of the noise in each local update of each client. 
Similarly, \cref{eq:feb28-3-1}, the expectation is w.r.t. $\{\mathcal{B}^{(i)}_{k, t}\}_{t=0}^{\tau-1}$ and it follows due to the independence of the noise in each local update. Also, \cref{eq:feb28-3-2} and \cref{eq:feb28-3-3} follows since both downlink and uplink noises have zero mean.

Next, using the $L$-smoothness of $f$ and \cref{eq:feb28-1}, we get
\begin{flalign}
    \label{eq:feb28-4-0-1}
    \mathbb{E}[f(\bm{w}_{k+1})] & \leq 
    \mathbb{E}[f(\bm{w}_k)] + (A) + (B)
\end{flalign}
where
\begin{align}
    A = - \mathbb{E}\Big[ \Big\langle \nabla f(\bm{w}_k), \eta_k\Big[\frac{1}{r}\sum_{i \in \mathcal{S}_k}\Big( \bm{e}^{(i)}_k +  \sum_{\tau=0}^{E-1}\widehat{\bm{u}}_{k,\tau}^{(i)} + \widetilde\nabla {f}_i(\bm{w}_k + \bm{\nu}_k^{(i)}; \mathcal{B}^{(i)}_{k, 0}) - \widetilde\nabla {f}_i(\bm{w}_k ; \mathcal{B}^{(i)}_{k, 0})\Big)\Big] \Big\rangle\Big]
\end{align}
and
\begin{align}
    B = \frac{L}{2}\mathbb{E}\Big[\Big\|\eta_k\Big[\frac{1}{r}\sum_{i \in \mathcal{S}_k}\Big( \bm{e}^{(i)}_k +  \sum_{\tau=0}^{E-1}\widehat{\bm{u}}_{k,\tau}^{(i)} + \widetilde\nabla {f}_i(\bm{w}_k + \bm{\nu}_k^{(i)}; \mathcal{B}^{(i)}_{k, 0}) - \widetilde\nabla {f}_i(\bm{w}_k ; \mathcal{B}^{(i)}_{k, 0})\Big)\Big]\Big\|^2\Big]
\end{align}
Now using (A):
\begin{multline}
    A = \underbrace{-\eta_k\mathbb{E}\Big[ \Big\langle \nabla f(\bm{w}_k), \frac{1}{r}\sum_{i \in \mathcal{S}_k} \bm{e}^{(i)}_k \Big\rangle\Big]}_{(A_1)} \underbrace{-\eta_k\mathbb{E}\Big[ \Big\langle \nabla f(\bm{w}_k), \frac{1}{r}\sum_{i \in \mathcal{S}_k} \sum_{\tau=0}^{E-1}\widehat{\bm{u}}_{k,\tau}^{(i)} \Big\rangle\Big]}_{(A_2)}
    \\
    \underbrace{-\eta_k\mathbb{E}\Big[ \Big\langle \nabla f(\bm{w}_k), \frac{1}{r}\sum_{i \in \mathcal{S}_k} \big(\widetilde\nabla {f}_i(\bm{w}_k + \bm{\nu}_k^{(i)}; \mathcal{B}^{(i)}_{k, 0}) - \widetilde\nabla {f}_i(\bm{w}_k ; \mathcal{B}^{(i)}_{k, 0})\big) \Big\rangle\Big]}_{(A_3)}
\end{multline}
$A_1$ will be zero since uplink noise has zero mean. Now, let's use $A_2:$
\begin{flalign}
    \nonumber
    A_2 & = -\eta_k\mathbb{E}\Big[ \Big\langle \nabla f(\bm{w}_k), \sum_{\tau=0}^{E-1}{\bm{u}}_{k,\tau} \Big\rangle\Big] \\
    \nonumber
    & = 
   -\eta_k\sum_{\tau=0}^{E-1}\mathbb{E}\Big[ \Big\langle \nabla f(\bm{w}_k),  {\bm{u}}_{k,\tau} \Big\rangle\Big]
\end{flalign}
For any 2 vectors $\bm{a}$ and $\bm{b}$, we have that:
\begin{flalign}
    \label{eq:feb28-103-2}
    \langle \bm{a}, \bm{b} \rangle = \frac{1}{2}(\|\bm{a}\|^2 + \|\bm{b}\|^2 - \|\bm{a} - \bm{b}\|^2)
\end{flalign}
Using this we will get $A_2$ as:
\begin{align}
    \label{eq:feb28-108}
    A_2= \frac{\eta_k}{2} \sum_{\tau=0}^{E-1} \mathbb{E}[\|\nabla f(\bm{w}_k) - \bm{u}_{k,\tau}\|^2] 
    - \frac{\eta_k E}{2} \mathbb{E}[\|\nabla f(\bm{w}_k)\|^2] - \frac{\eta_k}{2} \sum_{\tau=0}^{E-1} \mathbb{E}[\|\bm{u}_{k,\tau}\|^2]
\end{align}
Again using $A_3:$
\begin{flalign*}
    -\eta_k\mathbb{E}\Big[ \Big\langle \nabla f(\bm{w}_k), \frac{1}{n}\sum_{i \in [n]} (\widetilde\nabla {f}_i(\bm{w}_k + \bm{\nu}_k^{(i)}) - \widetilde\nabla {f}_i(\bm{w}_k) \Big\rangle\Big]
    = -\eta_k\mathbb{E}\Big[ \Big\langle \nabla f(\bm{w}_k), \frac{1}{n}\sum_{i \in [n]} (\nabla {f}_i(\bm{w}_k + \bm{\nu}_k^{(i)})\Big\rangle\Big]
    \\  
    + \eta_k\mathbb{E}\Big[ \Big\langle \nabla f(\bm{w}_k), \frac{1}{n}\sum_{i \in [n]} \nabla {f}_i(\bm{w}_k) \Big\rangle\Big]
\end{flalign*}
Using \cref{eq:feb28-103-2} and \cref{eq:feb28-101} in the equation above, we get
\begin{flalign*}
    A_3 = \frac{\eta_k}{2} \underbrace{\mathbb{E}[\|\nabla f(\bm{w}_k) - \frac{1}{n}\sum_{i \in [n]}\nabla{f}_{i}(\bm{w}_k + \bm{n}^{i}_k)\|^2]}_{A^{,}_3} 
    - \frac{\eta_k}{2} \mathbb{E}[\|\nabla f(\bm{w}_k)\|^2] - \frac{\eta_k}{2} \mathbb{E}[\|\frac{1}{n}\sum_{i \in [n]}\nabla{f}_{i}(\bm{w}_k + \bm{n}^{i}_k)\|^2]
    \\+ \eta_k \mathbb{E}[\langle \nabla{f}(\bm{w}_k), \nabla{f}(\bm{w}_k)\rangle]
\end{flalign*}
Reducing $A_3^{,}:$
\begin{flalign}
    \label{eq:feb28-104}
    A_3^{,} &= 
    % \label{eq:feb28-105}
    \frac{\eta_k}{2} \mathbb{E}[\|\frac{1}{n}\sum_{i \in [n]}\nabla{f}_{i}(\bm{w}_k) - \frac{1}{n}\sum_{i \in [n]}\nabla{f}_{i}(\bm{w}_k + \bm{n}^{i}_k)\|^2]
    \\
    \label{eq:feb28-106}
    &\leq
    \frac{\eta_k L^{2}}{2} \frac{1}{n}\sum_{i \in [n]} \bm{N}^{2}_{k, i}
\end{flalign}
The \cref{eq:feb28-104} follows by using \cref{eq:feb28-101}, while \cref{eq:feb28-106} follows from the $L$-smoothness of $f_i$, \cref{eq:feb28-3-2} and independence of noises. So, $A_3$ now becomes:
\begin{align}
\nonumber
    A_3 &\leq \frac{\eta_k L^{2}}{2} \frac{1}{n}\sum_{i \in [n]} \bm{N}^{2}_{k, i} - \frac{\eta_k}{2} \mathbb{E}[\|\nabla f(\bm{w}_k)\|^2] - \frac{\eta_k}{2} \mathbb{E}[\|\frac{1}{n}\sum_{i \in [n]}\nabla{f}_{i}(\bm{w}_k + \bm{n}^{i}_k)\|^2] + \eta_k \mathbb{E}[\|\nabla{f}(\bm{w}_k)\|^2]
    \\
    &\leq
    \label{eq:feb28-107}
    \frac{\eta_k L^{2}}{2} \frac{1}{n}\sum_{i \in [n]} \bm{N}^{2}_{k, i} + \frac{\eta_k}{2} \mathbb{E}[\|\nabla f(\bm{w}_k)\|^2] - \frac{\eta_k}{2} \mathbb{E}[\|\frac{1}{n}\sum_{i \in [n]}\nabla{f}_{i}(\bm{w}_k + \bm{n}^{i}_k)\|^2]
\end{align}
So, finally by combining \cref{eq:feb28-108} and \cref{eq:feb28-107}, A becomes:
\begin{multline}
    \label{eq:feb28-109}
    A \leq \frac{\eta_k}{2} \sum_{\tau=0}^{E-1} \mathbb{E}[\|\nabla f(\bm{w}_k) - \bm{u}_{k,\tau}\|^2] - \frac{\eta_k (E-1)}{2} \mathbb{E}[\|\nabla f(\bm{w}_k)\|^2] 
    - \frac{\eta_k}{2} \sum_{\tau=0}^{E-1} \mathbb{E}[\|\bm{u}_{k,\tau}\|^2] \\+ \frac{\eta_k L^{2}}{2} \frac{1}{n}\sum_{i \in [n]} \bm{N}^{2}_{k, i}
    - \frac{\eta_k}{2} \mathbb{E}[\|\frac{1}{n}\sum_{i \in [n]}\nabla{f}_{i}(\bm{w}_k + \bm{n}^{i}_k)\|^2]
\end{multline}
Now using (B):
\begin{multline}
    \label{eq:feb28-110}
    B = \frac{L}{2}\eta^{2}_k \Big(\mathbb{E}[\|\frac{1}{r}\sum_{i \in \mathcal{S}_k} \bm{e}^{(i)}_k\|^2] 
     + \mathbb{E}[\|\frac{1}{r}\sum_{i \in \mathcal{S}_k} \sum_{\tau=0}^{E-1}\widehat{\bm{u}}_{k,\tau}^{(i)}\|^2]  + \mathbb{E}[\|\frac{1}{r}\sum_{i \in \mathcal{S}_k} \widetilde\nabla {f}_i(\bm{w}_k + \bm{\nu}_k^{(i)}; \mathcal{B}^{(i)}_{k, 0}) - \widetilde\nabla {f}_i(\bm{w}_k ; \mathcal{B}^{(i)}_{k, 0})\|^2] 
     \\+2 \mathbb{E}[\langle \frac{1}{r}\sum_{i \in \mathcal{S}_k} \bm{e}^{(i)}_k, \frac{1}{r}\sum_{i \in \mathcal{S}_k} \sum_{\tau=0}^{E-1}\widehat{\bm{u}}_{k,\tau}^{(i)}\rangle]
    +2\mathbb{E}[\langle \frac{1}{r}\sum_{i \in \mathcal{S}_k} \bm{e}^{(i)}_k, \frac{1}{r}\sum_{i \in \mathcal{S}_k} (\widetilde\nabla {f}_i(\bm{w}_k + \bm{\nu}_k^{(i)}; \mathcal{B}^{(i)}_{k, 0}) - \widetilde\nabla {f}_i(\bm{w}_k ; \mathcal{B}^{(i)}_{k, 0})) \rangle]
    \\+2\mathbb{E}[\langle \frac{1}{r}\sum_{i \in \mathcal{S}_k} \sum_{\tau=0}^{E-1}\widehat{\bm{u}}_{k,\tau}^{(i)}, \frac{1}{r}\sum_{i \in \mathcal{S}_k} (\widetilde\nabla {f}_i(\bm{w}_k + \bm{\nu}_k^{(i)}; \mathcal{B}^{(i)}_{k, 0}) - \widetilde\nabla {f}_i(\bm{w}_k ; \mathcal{B}^{(i)}_{k, 0})\rangle]\Big)
\end{multline}
Using the fact that $\mathbb{E}[\bm{e}^{i}_k] = 0$ and Young's Inequality in \cref{eq:feb28-110}, we get:
\begin{multline}
    B \leq  \frac{L}{2}\eta^{2}_k \Big(\underbrace{\mathbb{E}[\|\frac{1}{r}\sum_{i \in \mathcal{S}_k} \bm{e}^{(i)}_k\|^2]}_{B_1} 
     + \underbrace{2\mathbb{E}[\|\frac{1}{r}\sum_{i \in \mathcal{S}_k} \sum_{\tau=0}^{E-1}\widehat{\bm{u}}_{k,\tau}^{(i)}\|^2]}_{B_2} + \underbrace{2\mathbb{E}[\|\frac{1}{r}\sum_{i \in \mathcal{S}_k} \widetilde\nabla {f}_i(\bm{w}_k + \bm{\nu}_k^{(i)}; \mathcal{B}^{(i)}_{k, 0}) - \widetilde\nabla {f}_i(\bm{w}_k ; \mathcal{B}^{(i)}_{k, 0})\|^2]}_{B_3} \Big)
\end{multline}
Starting with$B_1:$
\begin{flalign}
    \label{eq:feb28-111}
    B_1 & = \frac{n(r-1)}{r(n-1)} \mathbb{E}[\|\frac{1}{n}\sum_{i \in [n]} \bm{e}^{(i)}_k\|^2]
    + \frac{(n-r)}{r(n-1)}\frac{1}{n}\sum_{i \in [n]}\mathbb{E}[\|\bm{e}^{(i)}_k\|^2]
    \\
    \label{eq:feb28-112}
    & = \frac{(r-1)}{n r(n-1)}\sum_{i \in [n]}\bm{U}^2_{k, i} + \frac{(n-r)}{n r(n-1)}\sum_{i \in [n]}\bm{U}^2_{k, i}
    \\
    \label{eq:feb28-113}
    & = \frac{1}{n r}\sum_{i \in [n]}\bm{U}^2_{k, i}
\end{flalign}
Here, \cref{eq:feb28-111} follows due to expectation w.r.t $\mathcal{S}_k$ and \cref{eq:feb28-112} follows due to expectation w.r.t uplink noise and its independence.
Now let's focus on $B_2:$
\begin{align}
    \label{eq:feb28-114}
    B_2 \leq \frac{2n(r-1)E}{r(n-1)} \Big(\sum_{\tau=0}^{E-1}\mathbb{E}[\|{\bm{u}}_{k,\tau}\|^2] + \frac{\sigma^2}{n}\Big) + \frac{2(n-r)E}{r(n-1)} \Big(\frac{1}{n}\sum_{i \in [n]}\sum_{\tau=0}^{E-1}\mathbb{E}[\|{\nabla{f}_i(\bm{w}^{i}_{k,\tau})}\|^2] + \sigma^2\Big)
\end{align}
The \cref{eq:feb28-114} follows due to expectation w.r.t $\mathcal{S}_k$, \cref{eq:feb28-2}, \cref{eq:feb28-3} and \cref{eq:feb28-3-1}. Again, let's use $B_3:$
\begin{flalign}
%   \label{eq:feb28-115}
    \label{eq:feb28-116}
   B_3
   &\leq 2\mathbb{E}[\frac{1}{r}\sum_{i \in \mathcal{S}_k}\| \widetilde\nabla {f}_i(\bm{w}_k + \bm{\nu}_k^{(i)}; \mathcal{B}^{(i)}_{k, 0}) - \widetilde\nabla {f}_i(\bm{w}_k ; \mathcal{B}^{(i)}_{k, 0})\|^2]
   \\
   \label{eq:feb28-117}
   &\leq 2L^2\mathbb{E}[\frac{1}{r}\sum_{i \in \mathcal{S}_k}\|\bm{\nu}_k^{(i)}\|^2]
   \\
   \label{eq:feb28-118}
   &\leq 2L^2\frac{1}{n}\sum_{i \in [n]}\bm{N}^{2}_{k,i}
\end{flalign}
Here we used Jensen's inequality to reach \cref{eq:feb28-116}. Again, \cref{eq:feb28-117} follows due to the $L$-smoothness of $f_i$ and \cref{eq:feb28-118} follows due to expectation w.r.t $\mathcal{S}_k$ and \cref{eq:feb28-3-2}. So, finally by combining $B_1$, $B_2$ and $B_3$, (B) becomes:
\begin{multline}
    \label{eq:feb28-119}
    B \leq \frac{L\eta^2_k}{2}\Big(\frac{1}{nr}\sum_{i \in [n]}\bm{U}^2_{k, i}
    + \frac{2n(r-1)E}{r(n-1)} \big(\sum_{\tau=0}^{E-1}\mathbb{E}[\|{\bm{u}}_{k,\tau}\|^2] + \frac{\sigma^2}{n}\big) + \frac{2(n-r)E}{r(n-1)} \big(\frac{1}{n}\sum_{i \in [n]}\sum_{\tau=0}^{E-1}\mathbb{E}[\|{\nabla{f}_i(\bm{w}^{i}_{k,\tau})}\|^2] + \sigma^2\big) 
    \\+ 2L^2\frac{1}{n}\sum_{i \in [n]}\bm{N}^{2}_{k,i}\Big)
    \\
    \leq
    \frac{L\eta^2_k}{2nr}\sum_{i \in [n]}\bm{U}^2_{k, i} 
    + \frac{L^3\eta^2_k}{n}\sum_{i \in [n]}\bm{N}^{2}_{k, i}
    + \frac{\eta^{2}_kLE}{r}\sigma^2 + \frac{n(r-1)}{r(n-1)}\eta^{2}_kLE \Big(\sum_{\tau=0}^{E-1}\mathbb{E}[\|{\bm{u}}_{k,\tau}\|^2]\Big) 
    \\+ \frac{(n-r)}{r(n-1)}\eta^{2}_kLE \Big(\frac{1}{n}\sum_{i \in [n]}\sum_{\tau=0}^{E-1}\mathbb{E}[\|{\nabla{f}_i(\bm{w}^{i}_{k,\tau})}\|^2]\Big)
\end{multline}
Now, by putting \cref{eq:feb28-109} and \cref{eq:feb28-119} in \cref{eq:feb28-4-0-1} we will get:
\begin{multline}
    \label{eq:feb28-120}
    \mathbb{E}[f(\bm{w}_{k+1})] 
    \leq \mathbb{E}[f(\bm{w}_k)] - \frac{\eta_k (E-1)}{2} \mathbb{E}[\|\nabla f(\bm{w}_k)\|^2]
    -\frac{\eta_k}{2}\Big(1 -\eta_k L E \frac{n(r-1)}{r(n-1)} \Big)
    \sum_{\tau=0}^{E-1} \mathbb{E}[\|{\bm{u}}_{k,\tau}\|^2] + \frac{\eta^2_k L}{2nr}\sum_{i \in [n]}\bm{U}^2_{k,i}
    \\ - \frac{\eta_k}{2}\mathbb{E}[\|\frac{1}{n}\sum_{i \in [n]}\nabla{f}_{i}(\bm{w}_k + \bm{\nu}_k^{(i)})\|^2] + \frac{\eta^{2}_kLE}{r}\sigma^2
    + \frac{\eta_k L^2}{2}\Big(1 + 2\eta_kL\Big)\frac{1}{n}\sum_{i \in [n]}\bm{N}^2_{k,i}
    + \underbrace{\frac{\eta_k}{2} \sum_{\tau=0}^{E-1} \mathbb{E}[\|\nabla f(\bm{w}_k) - \bm{u}_{k,\tau}\|^2]}_{(M)}
    \\+ \underbrace{\eta_k^2 L E \frac{(n-r)}{r(n-1)}\Big(\frac{1}{n}\sum_{i \in [n]}\sum_{\tau=0}^{E-1} \mathbb{E}[\|\nabla {f}_i(\bm{w}^{(i)}_{k,\tau})\|^2]\Big)}_{(N)}
\end{multline}
We upper bound (M) and (N) using \Cref{sept26-lem2} and \Cref{sept26-lem1}, respectively. Plugging in these bounds and dropping the last term of \cref{eq:feb28-120}, we get:
\begin{multline}
    \label{eq:sep26-15}
    \mathbb{E}[f(\bm{w}_{k+1})] 
    \leq \mathbb{E}[f(\bm{w}_k)] - \frac{\eta_k (E-1)}{2} \mathbb{E}[\|\nabla f(\bm{w}_k)\|^2] 
    - \frac{\eta_k}{2}
    \underbrace{\Big(1 - \eta_k L E \frac{n(r-1)}{r(n-1)} - 2\eta_k^2 L^2 E^2 \Big)}_\text{(C)}
    \sum_{\tau=0}^{E-1} \mathbb{E}[\|{\bm{u}}_{k,\tau}\|^2]
    \\
    + 4\eta_k^2 L E^2 \Big(\frac{(n-r)}{r(n-1)} + \frac{2}{3}\eta_k L E\Big)\Big(\frac{1}{n}\sum_{i \in [n]} \mathbb{E}[\|\nabla {f}_i(\bm{w}_{k})\|^2]\Big) 
    + \eta_k^2 L E \Big(\frac{\eta_k L E}{n}\Big(1 + \frac{2nE}{3} + n\Big) + \frac{1}{r} + \frac{(n-r)}{r(n-1)}\Big)\sigma^2
    \\
    + \frac{\eta^2_kL}{2r}\frac{1}{n}\sum_{i \in [n]}\bm{U}^2_{k,i}
    + \frac{\eta_kL^2}{2}\Big(1 + 2\eta_kL + 4E\{1+3\eta^2_kL^2 + 2\eta_kLE(2+3\eta^2_kL^2)(\frac{2}{3}\eta_kLE +\frac{(n-r)}{r(n-1)})\}\Big)\frac{1}{n}\sum_{i \in [n]}\bm{N}^2_{k,i}.
\end{multline}
for $\eta_k L E \leq \frac{1}{2}$. Note that $\text{(C)} \geq 0$ for $\eta_k L E \leq \frac{1}{2}$. Thus, for $\eta_k L E \leq \frac{1}{2}$, we have:
\begin{multline}
    \label{eq:sep26-16}
    \mathbb{E}[f(\bm{w}_{k+1})] 
    \leq \mathbb{E}[f(\bm{w}_k)] - \frac{\eta_k (E-1)}{2} \mathbb{E}[\|\nabla f(\bm{w}_k)\|^2]+ 4\eta_k^2 L E^2 \Big(\frac{(n-r)}{r(n-1)} + \frac{2}{3}\eta_k L E\Big)\Big(\frac{1}{n}\sum_{i \in [n]} \mathbb{E}[\|\nabla {f}_i(\bm{w}_{k})\|^2]\Big) 
    \\
    + \eta_k^2 L E \Big(\frac{\eta_k L E}{n}\Big(1 + \frac{2nE}{3} + n\Big) + \frac{1}{r} + \frac{(n-r)}{r(n-1)}\Big)\sigma^2
    + \frac{\eta^2_kL}{2r}\frac{1}{n}\sum_{i \in [n]}\bm{U}^2_{k,i}
    + \frac{\eta_kL^2}{2}\Big(1 + 2\eta_kL + 4E\{1+3\eta^2_kL^2
    \\
    + 2\eta_kLE(2+3\eta^2_kL^2)(\frac{2}{3}\eta_kLE +\frac{(n-r)}{r(n-1)})\}\Big)\frac{1}{n}\sum_{i \in [n]}\bm{N}^2_{k,i}.
\end{multline}
\end{proof}
\begin{lemma}
\label{sept26-lem2}
For $\eta_k L E \leq \frac{1}{2}$:
\begin{multline*}
    \sum_{\tau=0}^{E-1}\mathbb{E}[\|\nabla f(\bm{w}_k) - \bm{u}_{k,\tau}\|^2] \leq2\eta_k^2L^2E^2\sum_{\tau=0}^{E-1}\mathbb{E}[\|\bm{u}_{k,\tau}\|^2]
    + \frac{16}{3}\eta_k^2L^2E^3\frac{1}{n}\sum_{i \in [n]}\mathbb{E}[\|\nabla {f}_i(\bm{w}_{k})\|^2] 
    + 2\eta_k^2L^2E^2(\frac{1}{n}+\frac{2E}{3} +1)\sigma^2
    \\+ 4L^2E(1+3\eta^2_kL^2 + \frac{4}{3}\eta^2_kL^2E^2(2+3\eta^2_kL^2))\frac{1}{n}\sum_{i \in [n]}\bm{N}^2_{k,i}.
\end{multline*}
\end{lemma}
\begin{proof}
We have:
\begin{align}
    \nonumber
    \mathbb{E}[\|\nabla f(\bm{w}_k) - \bm{u}_{k,\tau}\|^2] &= \mathbb{E}[\|\nabla f(\bm{w}_k) - \nabla f(\overline{\bm{w}}_{k,\tau}) + \nabla f(\overline{\bm{w}}_{k,\tau}) - \bm{u}_{k,\tau}\|^2]
    \\
    \nonumber
    &\leq 2\mathbb{E}[\|\nabla f(\bm{w}_k) - \nabla f(\overline{\bm{w}}_{k,\tau})\|^2] + 2\mathbb{E}[\|\nabla f(\overline{\bm{w}}_{k,\tau}) - \bm{u}_{k,\tau}\|^2]
    \\
    \label{eq:sep26-1}
    &\leq \underbrace{2L^2\mathbb{E}[\|\bm{w}_k - \overline{\bm{w}}_{k,\tau}\|^2]}_{M_1} + \underbrace{ 2\mathbb{E}\Big[\Big\|\frac{1}{n}\sum_{i \in [n]} (\nabla f_i(\overline{\bm{w}}_{k,\tau})- \nabla f_i(\bm{w}^{(i)}_{k, \tau}))\Big\|^2\Big]}_{M_2}
\end{align}
Using $M_1:$
\begin{multline}
    \label{eq:sep26-2}
    2L^2\mathbb{E}[\|\bm{w}_k - \overline{\bm{w}}_{k,\tau}\|^2] = 2L^2\mathbb{E}\Big[\Big\|\eta_k\sum_{t=0}^{\tau-1}\widehat{\bm{u}}_{k,t} + \frac{\eta_k}{n}\sum_{i \in [n]}\{ \widetilde\nabla {f}_i(\bm{w}_k + \bm{\nu}_k^{(i)}; \mathcal{B}^{(i)}_{k, 0}) - \widetilde\nabla {f}_i(\bm{w}_k ; \mathcal{B}^{(i)}_{k, 0})\} - \frac{1}{n}\sum_{i \in [n]}\bm{\nu}_k^{(i)}\Big\|^2\Big]
    \\
    = 2L^2\Big(\mathbb{E}[\|\eta_k\sum_{t=0}^{\tau-1}\widehat{\bm{u}}_{k,t}\|^2] + \mathbb{E}[\|\frac{1}{n}\sum_{i \in [n]}\bm{\nu}_k^{(i)}\|^2]
    + \mathbb{E}[\|\frac{\eta_k}{n}\sum_{i \in [n]}\{ \widetilde\nabla {f}_i(\bm{w}_k + \bm{\nu}_k^{(i)}; \mathcal{B}^{(i)}_{k, 0}) - \widetilde\nabla {f}_i(\bm{w}_k ; \mathcal{B}^{(i)}_{k, 0})\}\|^2]
    \\+2\mathbb{E}[\langle\eta_k\sum_{t=0}^{\tau-1}\widehat{\bm{u}}_{k,t}, \frac{\eta_k}{n}\sum_{i \in [n]}\{ \widetilde\nabla {f}_i(\bm{w}_k + \bm{\nu}_k^{(i)}; \mathcal{B}^{(i)}_{k, 0}) - \widetilde\nabla {f}_i(\bm{w}_k ; \mathcal{B}^{(i)}_{k, 0})\} \rangle]
    \\+2\mathbb{E}[\langle\frac{\eta_k}{n}\sum_{i \in [n]}\{ \widetilde\nabla {f}_i(\bm{w}_k + \bm{\nu}_k^{(i)}; \mathcal{B}^{(i)}_{k, 0}) - \widetilde\nabla {f}_i(\bm{w}_k ; \mathcal{B}^{(i)}_{k, 0})\}, - \frac{1}{n}\sum_{i \in [n]}\bm{\nu}_k^{(i)} \rangle]
    +2\mathbb{E}[\langle - \frac{1}{n}\sum_{i \in [n]}\bm{\nu}_k^{(i)},\eta_k\sum_{t=0}^{\tau-1}\widehat{\bm{u}}_{k,t}\rangle]
    \Big)
\end{multline}
Simplifying \cref{eq:sep26-2} using the fact that $\mathbb{E}[\|\bm{n}^{(i)}_{k}\|]=0$ and Young's Inequality we will get:
\begin{multline*}
\nonumber
    2L^2\mathbb{E}[\|\bm{w}_k - \overline{\bm{w}}_{k,\tau}\|^2] \leq 2L^2\Big(2\mathbb{E}[\|\eta_k\sum_{t=0}^{\tau-1}\widehat{\bm{u}}_{k,t}\|^2] 
    + 3\mathbb{E}[\|\frac{\eta_k}{n}\sum_{i \in [n]}\{ \widetilde\nabla {f}_i(\bm{w}_k + \bm{\nu}_k^{(i)}; \mathcal{B}^{(i)}_{k, 0}) - \widetilde\nabla {f}_i(\bm{w}_k ; \mathcal{B}^{(i)}_{k, 0})\}\|^2] \\
    + 2\mathbb{E}[\|\frac{1}{n}\sum_{i \in [n]}\bm{\nu}_k^{(i)}\|^2]\Big)
\end{multline*}
Using Jensen's Inequality in the equation above, we get
\begin{multline}
    2L^2\mathbb{E}[\|\bm{w}_k - \overline{\bm{w}}_{k,\tau}\|^2] \leq 2L^2\Big(2\eta_k^2\mathbb{E}[\|\sum_{t=0}^{\tau-1}\widehat{\bm{u}}_{k,t}\|^2] 
    + 3\eta_k^2\frac{1}{n}\sum_{i \in [n]}\mathbb{E}[\|\widetilde\nabla {f}_i(\bm{w}_k + \bm{\nu}_k^{(i)}; \mathcal{B}^{(i)}_{k, 0}) - \widetilde\nabla {f}_i(\bm{w}_k ; \mathcal{B}^{(i)}_{k, 0})\|^2]\\
    \label{eq:sep26-4}
    + \frac{2}{n^2}\mathbb{E}[\|\sum_{i \in [n]}\bm{\nu}_k^{(i)}\|^2]\Big)
\end{multline}
In \cref{eq:sep26-4} using the $L$-smoothness of $f_i$, \cref{eq:feb28-3}, \cref{eq:feb28-3-2} and independence of noise, we get
\begin{align*}
    2L^2\mathbb{E}[\|\bm{w}_k - \overline{\bm{w}}_{k,\tau}\|^2] \leq 2L^2\Big(2\eta_k^2\tau(\sum_{t=0}^{\tau-1}\mathbb{E}[\|\bm{u}_{k,t}\|^2] + \frac{\sigma^2}{n}) + 3\eta_k^2L^2\frac{1}{n}\sum_{i \in [n]}\bm{N}_{k,i}^2 + \frac{2}{n^2}\sum_{i \in [n]}\bm{N}_{k,i}^2\Big)
\end{align*}
Now let's use $M_2$:
\begin{align}
    \label{eq:sep26-5}
    2\mathbb{E}[\|\frac{1}{n}\sum_{i \in [n]} (\nabla
    f_i(\overline{\bm{w}}_{k,\tau})- \nabla f_i(\bm{w}^{(i)}_{k, \tau}))\|^2]
     \leq
    2L^2\frac{1}{n}\sum_{i \in [n]}\mathbb{E}[\|\bm{w}^{(i)}_{k,
    \tau}-\overline{\bm{w}}_{k,\tau}\|^2]
\end{align}
Here \cref{eq:sep26-5} follows from Jensen's Inequality. Now using the definition of $\overline{\bm{w}}_{k,\tau}$ and the $L$-smoothness of $f_i$ in conjunction with the same simplification process as used to simplify \cref{eq:sep26-2}, we get: 
\begin{multline}
    \label{eq:sep26-7}
    2\mathbb{E}[\|\frac{1}{n}\sum_{i \in [n]} (\nabla
    f_i(\overline{\bm{w}}_{k,\tau})- \nabla f_i(\bm{w}^{(i)}_{k, \tau}))\|^2]
     \leq
    2L^2\frac{1}{n}\sum_{i \in [n]}\Big(\underbrace{2\mathbb{E}[\|(\bm{n}_{k}^{(i)} - \frac{1}{n}\sum_{i \in [n]} \bm{n}_{k}^{(i)})\|^2]}_{(X)}
    +
    \underbrace{2\mathbb{E}[\|\eta_k(\sum_{t=0}^{\tau-1} \widehat{\bm{u}}_{k,t} - \sum_{t=0}^{\tau-1} \widehat{\bm{u}}_{k,t}^{(i)})\|^2]}_{(Y)}
    \\
    \underbrace{+3\mathbb{E}[\|\eta_k(\{\frac{1}{n}\sum_{i \in [n]}(\widetilde\nabla {f}_i(\bm{w}_k + \bm{\nu}_k^{(i)}; \mathcal{B}^{(i)}_{k, 0}) - \widetilde\nabla {f}_i(\bm{w}_k ; \mathcal{B}^{(i)}_{k, 0}))\} - \{\widetilde\nabla {f}_i(\bm{w}_k + \bm{\nu}_k^{(i)}; \mathcal{B}^{(i)}_{k, 0}) 
    - \widetilde\nabla {f}_i(\bm{w}_k ; \mathcal{B}^{(i)}_{k, 0})\})\|^2]\Big)}_{(Z)}
\end{multline}
Using (X):
\begin{gather}
    \nonumber
    2\mathbb{E}[\|(\bm{n}_{k}^{(i)} - \frac{1}{n}\sum_{i \in [n]} \bm{n}_{k}^{(i)})\|^2] = 2\mathbb{E}[\|\bm{n}_{k}^{(i)}\|^2] 
    + 2\mathbb{E}[\|\frac{1}{n}\sum_{i \in [n]}\bm{n}_{k}^{(i)}\|^2] 
    - 4\mathbb{E}[\langle \bm{n}_{k}^{(i)}, \frac{1}{n}\sum_{i \in [n]} \bm{n}_{k}^{(i)} \rangle]
    \\
    \label{eq:sep26-8}
    = 2 \bm{N}^2_{k,i} + \frac{2}{n^2}\sum_{i \in [n]}\bm{N}^2_{k,i} - \frac{4}{n}\bm{N}^2_{k,i}
    \\
    \label{eq:sep26-9}
    = 2(1- \frac{2}{n})\bm{N}^2_{k,i} + \frac{2}{n^2}\sum_{i \in [n]}\bm{N}^2_{k,i}
\end{gather}
\Cref{eq:sep26-8} follows due to \cref{eq:feb28-3-2} and independence of noise. So, now moving on to (Y): 
\begin{flalign}
    \nonumber
    2\mathbb{E}[\|\eta_k(\sum_{t=0}^{\tau-1} \widehat{\bm{u}}_{k,t} - \sum_{t=0}^{\tau-1} \widehat{\bm{u}}_{k,t}^{(i)})\|^2]
    & = 2\eta_k^2 \mathbb{E}[\| \sum_{t=0}^{\tau-1} \widehat{\bm{u}}_{k,t} - \sum_{t=0}^{\tau-1} \widehat{\bm{u}}_{k,t}^{(i)}\|^2] 
    \\
    \label{eq:sept26-10}
    & \leq 2\eta_k^2 \tau \sum_{t=0}^{\tau-1} \mathbb{E}[\|\widehat{\bm{u}}_{k,t} - \widehat{\bm{u}}_{k,t}^{(i)}\|^2]
    \\
    \label{eq:sept26-11}
    & = 2\eta_k^2 \tau \sum_{t=0}^{\tau-1} \mathbb{E}[\|\widehat{\bm{u}}_{k,t}\|^2 + \|\widehat{\bm{u}}_{k,t}^{(i)}\|^2 - 2\langle \widehat{\bm{u}}_{k,t}, \widehat{\bm{u}}_{k,t}^{(i)} \rangle]
\end{flalign}
\Cref{eq:sept26-10} follows because of Jensen's Inequality and using the fact that $\widehat{\bm{u}}_{k,\tau} = \frac{1}{n}\sum_{i \in [n]} \widehat{\bm{u}}_{k,\tau}^{(i)}$, we can simplify \cref{eq:sept26-11} to:
\begin{flalign}
    \nonumber
    2\mathbb{E}[\|\eta_k(\sum_{t=0}^{\tau-1} \widehat{\bm{u}}_{k,t} - \sum_{t=0}^{\tau-1} \widehat{\bm{u}}_{k,t}^{(i)})\|^2] 
    & \leq 2\eta_k^2 \tau \sum_{t=0}^{\tau-1} (\mathbb{E}[\|\widehat{\bm{u}}_{k,\tau}^{(i)}\|^2] - \mathbb{E}[\|\widehat{\bm{u}}_{k,t}\|^2])
    \\
    & \leq 2\eta_k^2 \tau \sum_{t=0}^{\tau-1} \mathbb{E}[\|{\widehat{\bm{u}}_{k,\tau}^{(i)}}\|^2]
    \\
    \label{eq:sept26-12}
    & \leq 2\eta_k^2 \tau \sum_{t=0}^{\tau-1}(\mathbb{E}[\|\nabla {f}_i(\bm{w}^{(i)}_{k, t})\|^2] + \sigma^2)
\end{flalign}
Now using \Cref{sept26-lem1} for $\bm{n}_kLE\leq \frac{1}{2}$ in \cref{eq:sept26-12}, we get:
\begin{flalign}
    \label{eq:sept26-12-1}
    2\mathbb{E}[\|\eta_k(\sum_{t=0}^{\tau-1} \widehat{\bm{u}}_{k,t} - \sum_{t=0}^{\tau-1} \widehat{\bm{u}}_{k,t}^{(i)})\|^2]
    \leq 2\eta_k^2\tau^2(4\mathbb{E}[\|\nabla{f}_i(\bm{w}_k)\|^2]  
    + 4L^2(2+3\eta_k^2L^2)\bm{N}^2_{k,i}) + 2\eta^2_k(\tau^2+\tau)\sigma^2
\end{flalign}
Next, using (Z):
\begin{multline}
    \label{eq:sept26-13}
    Z = 3\eta_k^2\mathbb{E}\Big[\|\frac{1}{n}\sum_{i \in [n]}(\widetilde\nabla {f}_i(\bm{w}_k + \bm{\nu}_k^{(i)}; \mathcal{B}^{(i)}_{k, 0}) - \widetilde\nabla {f}_i(\bm{w}_k ; \mathcal{B}^{(i)}_{k, 0}))\|^2 
    + \|(\widetilde\nabla {f}_i(\bm{w}_k + \bm{\nu}_k^{(i)}; \mathcal{B}^{(i)}_{k, 0}) - \widetilde\nabla {f}_i(\bm{w}_k ; \mathcal{B}^{(i)}_{k, 0}))\|^2 
    \\
    - 2\big\langle\frac{1}{n}\sum_{i \in [n]}(\widetilde\nabla {f}_i(\bm{w}_k + \bm{\nu}_k^{(i)}; \mathcal{B}^{(i)}_{k, 0}) - \widetilde\nabla {f}_i(\bm{w}_k ; \mathcal{B}^{(i)}_{k, 0})), (\widetilde\nabla {f}_i(\bm{w}_k + \bm{\nu}_k^{(i)}; \mathcal{B}^{(i)}_{k, 0}) - \widetilde\nabla {f}_i(\bm{w}_k ; \mathcal{B}^{(i)}_{k, 0}))\big\rangle\Big]
\end{multline}
We simplify \cref{eq:sept26-13} using the similar fact that we used to simplify \cref{eq:sept26-11}. Subsequently, by using the $L$-smoothness of $f_i$, we get
\begin{align}
    \nonumber
    Z &= 3\eta_k^2(\mathbb{E}[\|(\widetilde\nabla {f}_i(\bm{w}_k + \bm{\nu}_k^{(i)}; \mathcal{B}^{(i)}_{k, 0}) - \widetilde\nabla {f}_i(\bm{w}_k ; \mathcal{B}^{(i)}_{k, 0}))\|^2] 
    - \mathbb{E}[\|\frac{1}{n}\sum_{i \in [n]}(\widetilde\nabla {f}_i(\bm{w}_k + \bm{\nu}_k^{(i)}; \mathcal{B}^{(i)}_{k, 0}) - \widetilde\nabla {f}_i(\bm{w}_k ; \mathcal{B}^{(i)}_{k, 0}))\|^2])
    \\
    \nonumber
    &\leq
    3\eta_k^2\mathbb{E}[\|(\widetilde\nabla {f}_i(\bm{w}_k + \bm{\nu}_k^{(i)}; \mathcal{B}^{(i)}_{k, 0}) - \widetilde\nabla {f}_i(\bm{w}_k ; \mathcal{B}^{(i)}_{k, 0}))\|^2]
    \\
    \label{eq:sept26-14}
    &\leq
    3\eta_k^2L^2\mathbb{E}[\|\bm{\nu}_k^{(i)}\|^2]
    \leq
    3\eta_k^2L^2\bm{N}^2_{k,i}
\end{align}
Now putting the results of \cref{eq:sep26-9}, \cref{eq:sept26-12-1} and \cref{eq:sept26-14} in \cref{eq:sep26-7} we get:
\begin{multline}
    \label{eq:sept26-15}
    2\mathbb{E}[\|\frac{1}{n}\sum_{i \in [n]} (\nabla
    f_i(\overline{\bm{w}}_{k,\tau})- \nabla f_i(\bm{w}^{(i)}_{k, \tau}))\|^2]
    \\
    \leq
    2L^2\frac{1}{n}\sum_{i \in [n]}\Big(2(1- \frac{2}{n})\bm{N}^2_{k,i} + \frac{2}{n^2}\sum_{i \in [n]}\bm{N}^2_{k,i} 
    + 2\eta_k^2\tau^2(4\mathbb{E}[\|\nabla{f}_i(\bm{w}_k)\|^2] + 4L^2(2+3\eta_k^2L^2)\bm{N}^2_{k,i}) + 2\bm{n}^2_k(\tau^2+\tau)\sigma^2 + 
    3\eta_k^2L^2\bm{N}^2_{k,i}
    \Big)
    \\
    \leq 16\eta_k^2L^2\tau^2 \frac{1}{n}\sum_{i \in [n]}\mathbb{E}[\|\nabla {f}_i(\bm{w}{k})\|^2]+2L^2(2- \frac{2}{n}+3\eta_k^2L^2)\frac{1}{n}\sum_{i \in [n]}\bm{N}^2_{k,i} + 16\eta_k^2L^4\tau^2(2+3\eta_k^2L^2)\frac{1}{n}\sum_{i \in [n]}\bm{N}^2_{k,i}
    \\+ 4\eta_k^2L^2(\tau^2+\tau)\sigma^2.   
\end{multline}
So, by combining $M_1$ and $M_2$ in \cref{eq:sep26-1}, we get
\begin{multline}
    \label{eq:sept26-17}
    \mathbb{E}[\|\nabla f(\bm{w}_k) - \bm{u}_{k,\tau}\|^2] \leq 4\eta_k^2L^2\tau\sum_{t=0}^{\tau-1}\mathbb{E}[\|\bm{u}_{k,t}\|^2] + 16\eta_k^2L^2\tau^2 \frac{1}{n}\sum_{i \in [n]}\mathbb{E}[\|\nabla {f}_i(\bm{w}{k})\|^2]
    + 4\eta_k^2L^2(\tau^2+\frac{\tau}{n}+\tau)\sigma^2 
    \\
    + 4L^2(1+3\eta^2_kL^2)\frac{1}{n}\sum_{i \in [n]}\bm{N}^2_{k,i}  + 16\eta_k^2L^4\tau^2(2+3\eta^2_kL^2)\frac{1}{n}\sum_{i \in [n]}\bm{N}^2_{k,i}.
\end{multline}
Now summing up \cref{eq:sept26-17} for all $\tau \in \{0,\ldots,E-1\}$, we get:
\begin{multline}
    \label{eq:sept26-18}
    \sum_{\tau=0}^{E-1}\mathbb{E}[\|\nabla f(\bm{w}_k) - \bm{u}_{k,\tau}\|^2] \leq2\eta_k^2L^2E^2\sum_{\tau=0}^{E-1}\mathbb{E}[\|\bm{u}_{k,\tau}\|^2]
    + \frac{16}{3}\eta_k^2L^2E^3\frac{1}{n}\sum_{i \in [n]}\mathbb{E}[\|\nabla {f}_i(\bm{w}_{k})\|^2] 
    + 2\eta_k^2L^2E^2(\frac{1}{n}\\+\frac{2E}{3} +1)\sigma^2
    + 4L^2E\big(1+3\eta^2_kL^2 + \frac{4}{3}\eta^2_kL^2E^2(2+3\eta^2_kL^2)\big)\frac{1}{n}\sum_{i \in [n]}\bm{N}^2_{k,i}. 
\end{multline}
\end{proof}
\begin{lemma}
\label{sept26-lem1}
For $\eta_k L E \leq \frac{1}{2}$, we have:
\begin{align*}
    \sum_{t=0}^{\tau-1} \mathbb{E}[\|\nabla {f}_i(\bm{w}^{(i)}_{k, t})\|^2] \leq 
    4\tau \mathbb{E}[\|\nabla {f}_i(\bm{w}_k)\|^2] +
    4L^2\tau(2+3\eta^2_kL^2)\bm{N}_{k,i}^{2}
    + \sigma^2.
\end{align*}
\end{lemma}
\begin{proof}
\begin{align}
    \nonumber
    \mathbb{E}[\|\nabla {f}_i(\bm{w}^{(i)}_{k, t})\|^2] &= \mathbb{E}[\|\nabla {f}_i(\bm{w}^{(i)}_{k, t}) - \nabla {f}_i(\bm{w}_k) + \nabla {f}_i(\bm{w}_k)\|^2]
    \\
    \nonumber
    &\leq 2 \mathbb{E}[\|\nabla {f}_i(\bm{w}_k)\|^2] + 2 \mathbb{E}[\|\nabla {f}_i(\bm{w}^{(i)}_{k, t}) - \nabla {f}_i(\bm{w}_k)\|^2]
    \\
    \label{eq:feb28-7}
    &\leq 2 \mathbb{E}[\|\nabla {f}_i(\bm{w}_k)\|^2] + 2L^2 \mathbb{E}[\|\bm{w}^{(i)}_{k, t} - \bm{w}_k\|^2].
\end{align}
But: 
\begin{multline}
    \label{eq:feb28-8}
    \mathbb{E}[\|\bm{w}^{(i)}_{k, t} - \bm{w}_k\|^2]  = \mathbb{E}\Big[\Big\|\bm{\nu}_k^{(i)} - \eta_k\Big( \sum_{t'=0}^{t-1}\widehat{\bm{u}}_{k,t'}^{(i)} + \widetilde\nabla {f}_i(\bm{w}_k + \bm{\nu}_k^{(i)}; \mathcal{B}^{(i)}_{k, 0}) - \widetilde\nabla {f}_i(\bm{w}_k ; \mathcal{B}^{(i)}_{k, 0})\Big)\Big\|^2\Big]
    \\
    =  \mathbb{E}[\|\bm{\nu}_k^{(i)}\|^2] + \eta_k^2\mathbb{E}[\|\sum_{t'=0}^{t-1}\widehat{\bm{u}}_{k,t'}^{(i)}\|^2] 
    + \eta_k^2\mathbb{E}[\|\nabla\widetilde{f}_i(\bm{w}_k + \bm{\nu}_k^{(i)}; \mathcal{B}^{(i)}_{k, 0}) - \widetilde\nabla {f}_i(\bm{w}_k ; \mathcal{B}^{(i)}_{k, 0})\|^2]
    \\+ 2 \mathbb{E}[\langle \bm{\nu}_k^{(i)}, -\eta_k\sum_{t'=0}^{t-1}\widehat{\bm{u}}_{k,t'}^{(i)}\rangle]
    + 2 \mathbb{E}[\langle \bm{\nu}_k^{(i)}, -\eta_k\{\nabla\widetilde{f}_i(\bm{w}_k+ \bm{\nu}_k^{(i)} ; \mathcal{B}^{(i)}_{k, 0})- \widetilde\nabla {f}_i(\bm{w}_k ; \mathcal{B}^{(i)}_{k, 0})\}\rangle]
    \\+ 2 \mathbb{E}[\langle -\eta_k\sum_{t'=0}^{t-1}\widehat{\bm{u}}_{k,t'}^{(i)}, -\eta_k\{\nabla\widetilde{f}_i(\bm{w}_k+ \bm{\nu}_k^{(i)} ; \mathcal{B}^{(i)}_{k, 0})- \widetilde\nabla {f}_i(\bm{w}_k  ; \mathcal{B}^{(i)}_{k, 0})\}\rangle] 
\end{multline} 
Now using the fact that $\mathbb{E}[\bm{n}^{i}_{k}] = 0$ and Young's Inequality in \cref{eq:feb28-8}, we get:
\begin{align}
    \nonumber
    \mathbb{E}[\|\bm{w}^{(i)}_{k, t} - \bm{w}_k\|^2] &\leq 
    2\mathbb{E}[\|\bm{\nu}_k^{(i)}\|^2] 
    + 2\eta_k^2\mathbb{E}[\|\sum_{t'=0}^{t-1}\widehat{\bm{u}}_{k,t'}^{(i)}\|^2] 
     + 3\eta_k^2\mathbb{E}[\|\nabla\widetilde{f}_i(\bm{w}_k + \bm{\nu}_k^{(i)};
    \mathcal{B}^{(i)}_{k, 0}) - \widetilde\nabla {f}_i(\bm{w}_k ;
    \mathcal{B}^{(i)}_{k, 0})\|^2]
    \\
    \label{eq:feb28-9}
    &\leq 2\bm{N}_{k,i}^{2} 
    + 2\eta^2_k t(\sum_{t'=0}^{t-1}\mathbb{E}[\|\nabla{f}_{i}(\bm{w}^{(i)}_{k,t'}\|^2]) + \sigma^2) + 3\eta^2_kL^2\bm{N}_{k,i}^2
\end{align}
Putting \cref{eq:feb28-9} in \cref{eq:feb28-7}, we get:
\begin{align}
    \label{eq:feb28-10}
    \mathbb{E}[\|\nabla {f}_i(\bm{w}^{(i)}_{k, t})\|^2] \leq 
    2 \mathbb{E}[\|\nabla {f}_i(\bm{w}_k)\|^2] +
    2L^2(2+3\eta^2_kL^2)\bm{N}_{k,i}^{2} 
    + 4\eta^2_kL^2t(\sum_{t'=0}^{t-1}\mathbb{E}[\|\nabla{f}_{i}(\bm{w}^{(i)}_{k,t'})\|^2] + \sigma^2) 
\end{align}
Summing up \cref{eq:feb28-10} for all $t \in \{0,\ldots,\tau-1\}$, we get:
\begin{align}
    \sum_{t=0}^{\tau-1} \mathbb{E}[\|\nabla {f}_i(\bm{w}^{(i)}_{k, t})\|^2]  \leq 2\tau \mathbb{E}[\|\nabla {f}_i(\bm{w}_k)\|^2] +
    2L^2\tau(2+3\eta^2_kL^2)\bm{N}_{k,i}^{2} 
    + 2\eta^2_kL^2\tau^2(\sum_{t=0}^{\tau-1}\mathbb{E}[\|\nabla{f}_{i}(\bm{w}^{(i)}_{k,t})\|^2]+ \sigma^2)
\end{align}
Let us set $\eta_k L E \leq 1/2$. Then:
\begin{align}
    \label{eq:feb28-12}
    \sum_{t=0}^{\tau-1} \mathbb{E}[\|\nabla {f}_i(\bm{w}^{(i)}_{k, t})\|^2] \leq 2\tau \mathbb{E}[\|\nabla {f}_i(\bm{w}_k)\|^2] +
    2L^2\tau(2+3\eta^2_kL^2)\bm{N}_{k,i}^{2} 
    + \frac{1}{2}\sum_{t=0}^{\tau-1}\mathbb{E}[\|\nabla{f}_{i}(\bm{w}^{(i)}_{k,t})\|^2] + \frac{1}{2}\sigma^2
\end{align}
Simplifying, we get:
\begin{align}
    \label{eq:feb28-13}
    \sum_{t=0}^{\tau-1} \mathbb{E}[\|\nabla {f}_i(\bm{w}^{(i)}_{k, t})\|^2] \leq 
    4\tau \mathbb{E}[\|\nabla {f}_i(\bm{w}_k)\|^2] +
    4L^2\tau(2+3\eta^2_kL^2)\bm{N}_{k,i}^{2} 
    + \sigma^2
\end{align}
\end{proof}

\bibliographystyle{ieeetr}
\nocite{*}
% \bibliography{NF_IEEE_ACCESS/refs.bib}
\bibliography{Arxic/root_one_column}
\end{document}